\documentclass{article}

\usepackage[preprint,nonatbib]{neurips_2022}
\usepackage{rotating}

\usepackage[utf8]{inputenc} 
\usepackage[T1]{fontenc}    
\usepackage{hyperref}       
\usepackage{url}            
\usepackage{booktabs}       
\usepackage{amsfonts}       
\usepackage{nicefrac}       
\usepackage{microtype}      
\usepackage[dvipsnames]{xcolor}

\usepackage{amsmath}
\usepackage{amssymb}
\usepackage{mathtools}
\usepackage{amsthm}

\usepackage{enumitem}

\usepackage{amssymb}
\usepackage{pifont}
\newcommand{\cmark}{\textcolor{Green}{\ding{51}}}%
\newcommand{\xmark}{\textcolor{Red}{\ding{55}}}%

\usepackage{dsfont}

\usepackage[capitalize,noabbrev]{cleveref}

\theoremstyle{plain}

\theoremstyle{definition}

\theoremstyle{remark}

\usepackage[textsize=tiny]{todonotes}

\usepackage{overpic}
\usepackage{subcaption}
\usepackage{amsmath}
\usepackage{amsthm}

\newtheorem{prop}{Proposition}
\newtheorem{thm}{Theorem}
\newtheorem{lem}{Lemma}

\newcommand{\RR}{\mathbb{R}}
\newcommand{\vecc}{\boldsymbol}

\usepackage{mathtools}

\usepackage{stackengine}

\definecolor{dkgray}{rgb}{99,99,99}
\definecolor{darkred}{RGB}{228,26,28}
\definecolor{darkblue}{RGB}{44,127,184}
\definecolor{magentaCB}{RGB}{221,28,119}

\definecolor{morange}{RGB}{255, 187, 0}
\definecolor{mblue}{RGB}{ 0, 161, 241}

\newtheorem*{theorem*}{Theorem}

\makeatletter
\newcommand{\printfnsymbol}[1]{%
  \textsuperscript{\@fnsymbol{#1}}%
}
\makeatother

\usepackage{wrapfig}

\title{NoisyMix: Boosting Model Robustness to Common Corruptions}

\author{%
  N. Benjamin Erichson\thanks{Equal contribution. Contact: \url{erichson@pitt.edu} and \url{soon.hoe.lim@su.se}.} \\
  University of Pittsburgh\\
   \And
   Soon Hoe Lim$^*$ \\
   Nordita, KTH\\
   \And
   Winnie Xu \\
   University of Toronto \\
   \AND
   Francisco Utera \\
   University of Pittsburgh \\
   \And
   Ziang Cao \\
   University of Pittsburgh \\
   \And
	Michael W. Mahoney \\
	ICSI and UC Berkeley \\
}

\hypersetup{ %
	pdftitle={},
	pdfauthor={},
	pdfkeywords={},
	pdfborder=0 0 0,
	pdfpagemode=UseNone,
	colorlinks=true,
	linkcolor=magenta,
	citecolor=magenta,
	filecolor=magenta,
	urlcolor=magenta,
}

\begin{document}

\maketitle

\begin{abstract}
For many real-world applications, obtaining stable and robust statistical performance is more important than simply achieving state-of-the-art predictive test accuracy, and thus robustness of neural networks is an increasingly important topic. Relatedly, data augmentation schemes have been shown to improve robustness with respect to input perturbations and domain shifts. Motivated by this, we introduce \textit{NoisyMix}, a novel training scheme that promotes stability as well as leverages noisy augmentations in input and feature space to improve both model robustness and in-domain accuracy. \textit{NoisyMix} produces models that are consistently more robust and that provide well-calibrated estimates of class membership probabilities. We demonstrate the benefits of \textit{NoisyMix}  on a range of benchmark datasets, including ImageNet-C, ImageNet-R, and ImageNet-P. Moreover, we provide theory to understand implicit regularization and robustness of \textit{NoisyMix}.

\end{abstract}

\section{Introduction}
\label{submission}

While deep learning has achieved remarkable results in computer vision and natural language processing, much of the success is driven by improving test accuracy. However, it is well-known that deep learning models are typically brittle and sensitive to noisy and adversarial environments. 
This limits their applicability in many real-world problems which require, at a minimum, that deep learning methods produce stable statistical predictions.
One can distinguish between structural stability (i.e., how sensitive is a model's prediction to small perturbations of the weights, or model parameters) and input stability (i.e., how sensitive is a model's prediction to small perturbations of the input data). Structural stability is particularly important when compression techniques (e.g., quantization) or different hardware systems introduce small errors to the model's weights.
Input stability is important when models (e.g., a vision model in a self-driving car) should  operate reliably in noisy environments. 
Overall, robustness is an important topic within machine learning and deep learning, but the many facets of robustness prohibit the use of a single metric for measuring it~\cite{hendrycks2021many}. 

In this work, we are specifically interested in studying input stability (robustness) with respect to common data corruptions and domain shifts that naturally occur in many real-world applications.
To do so, we use datasets such as ImageNet-C~\cite{hendrycks2018benchmarking} and ImageNet-R~\cite{hendrycks2021many}.
ImageNet-C provides examples to test the robustness of a model with respect to common corruptions, including noise sources such as white noise, weather variations such as snow, and digital distortions such as image compression.
ImageNet-R provides examples to test robustness to naturally occurring domain shifts in form of abstract visual renditions, including graffiti, origami, or cartoons. 
State-of-the-art models often fail when facing such common corruptions and domain shifts. 
For instance, the predictive accuracy of a ResNet-50~\cite{he2016identity} trained on ImageNet drops by over $35\%$, when evaluated on ImageNet-C~\cite{hendrycks2018benchmarking}. 
%

%
\begin{wrapfigure}{r}{0.5\textwidth}
  \begin{center}\vspace{-0.2cm}
    \includegraphics[width=0.48\textwidth]{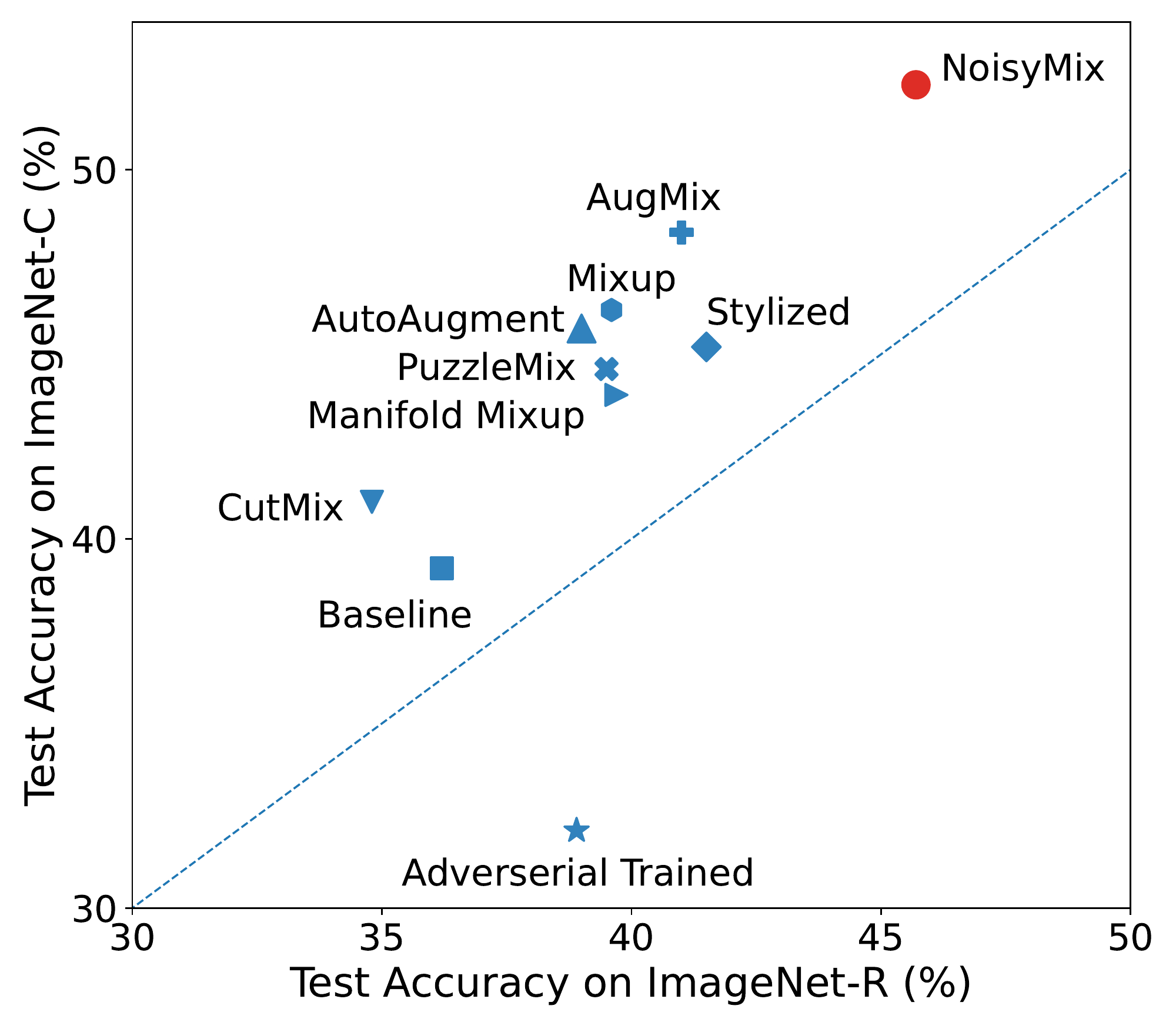}
  \end{center}
    \vspace{-0.4cm}
	\caption{ResNet-50 models trained with data augmentation methods. \textit{NoisyMix} considerably improves the test accuracy on ImageNet-C and ImageNet-R, indicating improved robustness to common corruptions and domain shifts.}\label{fig:robustCvsR}
\end{wrapfigure}
Data augmentation methods such as \textit{Mixup}~\cite{zhang2018mixup}, \textit{AutoAugment}~\cite{cubuk2019autoaugment}, training on \textit{stylized} ImageNet~\cite{geirhos2018imagenet}, and \textit{AugMix}~\cite{hendrycks2020augmix} have the ability to improve the resilience to various data perturbations and domain shifts.
Figure~\ref{fig:robustCvsR} shows the obtained test
accuracies on ImageNet-C and ImageNet-R for ResNet-50 models trained with different augmentation schemes on ImageNet (see Sec.~\ref{sec:results} for details). 
Clearly, 
data augmentation helps to achieve better test accuracies on these corrupted and domain-shifted test sets when compared to a ``vanilla'' trained model.

Among these, our proposed method, \textit{NoisyMix}, is the most effective in improving robustness.
In contrast, adversarial training~\cite{shafahi2019adversarial} decreases robustness to common corruptions. The reason for this is that while adversarial training helps to improve robustness to perturbations that are manifested in the high-frequency domain, it reduces robustness to common corruptions that are often manifested in the low-frequency domain~\cite{fourier}.

\textbf{Contributions.}  
Our main contributions in this paper are summarized as follows.

\begin{itemize}[leftmargin=*]
    \item We formulate and study a novel, theoretically justified training scheme, \textit{NoisyMix},  that substantially outperforms standard stability training and data augmentation schemes in improving robustness to common corruptions and in-domain test accuracy.

    \item We study \textit{NoisyMix} via the lens of implicit regularization and provide theoretical results that identify the  effects (see Theorem \ref{thm_implicitreg}-\ref{thm_jsd} in App.~\ref{sm_secA}). Importantly, we show that, under assumptions similar to those made in \cite{lamb2019interpolated,lim2021noisy}, minimizing the {\it NoisyMix} loss corresponds to minimizing an upper bound on the sum of an adversarial loss, a stability objective, and data-dependent regularizers (see Theorem \ref{thm_advbound} in App.~\ref{sm_secD}).
    This suggests that {\it NoisyMix} provides an inductive bias towards improving model robustness with respect to common corruptions, when compared to standard training.

    \item We provide empirical results to show that \textit{NoisyMix}  improves the robust accuracy on ImageNet-C and ImageNet-R by over $4\%$ compared to AugMix, while achieving comparable accuracy on the clean ImageNet validation set. Moreover, \textit{NoisyMix} has comparable training costs to AugMix. We also show that \textit{NoisyMix} achieves state-of-the-art performance on CIFAR-10-C and CIFAR-100-C, using preactivated ResNet-18 and Wide-ResNet-28x2 architectures.
\end{itemize}

\noindent {\bf Notation.}
$[K] := \{1,\dots, K\}$, $\odot$ denotes Hadamard product, $\mathds{1}$ denotes the vector with all components equal one. For a vector $v$, $\|v\|_p$ denotes its $l_p$ norm for $p > 0$. $M_\lambda(a,b) := \lambda a + (1-\lambda) b$, for  random variables $a, b, \lambda$.

\section{Related Work}\label{sec:related}

The study of stability of modelling methods for  real-world phenomena has a long history and is central in statistics, numerical analysis, and dynamical systems~\cite{yu2013stability}.
In this work, we are mainly concerned with improving input stability of deep neural networks for computer vision tasks.
In particular, there are two types of input perturbations that have recently led to a surge of interest: (i) adversarial perturbations, and (ii) common corruptions. 
The class of adversarial perturbations consists of artificially crafted perturbations specifically designed to fool neural networks while being imperceptible for the human eye~\cite{szegedy2013intriguing,goodfellow2014explaining}. 
Common corruptions are a class of perturbations that naturally occur in real-world applications~\cite{hendrycks2018benchmarking}, and  are often more relevant in real-world science and engineering applications than adversarial perturbations~\cite{rusak2020simple}.

Early work by \cite{dodge2016understanding} has revealed how sensitive neural networks (trained on high quality images) are to  slight degradation of image quality and common corruptions. 
\cite{hendrycks2018benchmarking} introduced a more extensive benchmark dataset for testing robustness on a comprehensive set of common corruptions with varying intensity levels. Subsequently, it has been shown that data augmentation and noise injection schemes can help to improve robustness~\cite{fourier} to common corruptions.
As an alternative, \cite{hendrycks2019using} and \cite{xie2020self} showed that pretraining on larger and more diverse datasets (e.g., JFT-300M~\cite{hinton2015distilling}) can help to improve robustness, but~\cite{hendrycks2021many} showed that these strategies do not consistently do better if a broader range of distribution shifts is considered.
The disadvantages of pretraining on larger datasets are the increased computational costs and the limited availability of massive computer vision datasets. Hence, we limit our discussion in this paper to data augmentation and noise injection strategies that can be applied to standard architectures trained on common computer vision benchmark datasets (e.g., ImageNet).

Most relevant for our work are data augmentation methods that aim to reduce the generalization error and improve robustness by introducing implicit regularization~\cite{Mah12,KGC17_TR} into the training process. 
By data augmentation, we refer to schemes that do not introduce new data but train a model with virtual training examples (e.g., proper transformations of the original data) in the vicinity of a given training dataset. Basic examples are random cropping and horizontal flipping of input images~\cite{krizhevsky2012imagenet}.

Mixup~\cite{zhang2018mixup} is a popular data augmentation method, which creates new training examples by forming linear interpolations between random pairs of examples and their corresponding labels. 
Despite its simplicity it is an extremely efficient scheme and helps to improve generalization and robustness of models with minimal computation overhead.
Motivated by Mixup, a range of other innovations have been proposed subsequently. 
Manifold Mixup~\cite{verma2019manifold} extends the idea of interpolating between data points in input space to hidden representations, leading to models that have smoother decision boundaries. 
Noisy Feature Mixup~\cite{lim2021noisy} generalizes the idea of both Mixup and Manifold Mixup by taking linear combinations of pairs of perturbed input and hidden representations.
Another approach involves zeroing out parts of an input image~\cite{devries2017improved} to prevent a model from focusing on a limited image region. 
CutMix~\cite{yun2019cutmix} extends this idea by replacing the removed parts with a patch from another image.
PuzzleMix~\cite{kim2020puzzle} is a mixup method that further improves CutMix by leveraging saliency information and local statistics of the original input data to decide how to mix two data points.

\cite{geirhos2018imagenet} proposed to train models on a stylized version of ImageNet, which helps to increase shape bias and in turn improves robustness to common corruptions. To do so, they use a style transfer network that applies artwork styles to images that are then used to train the model. Similar good results are achieved by AutoAugment~\cite{cubuk2019autoaugment}, a method that automatically searches for improved data augmentation policies.
AugMix~\cite{hendrycks2020augmix} is designed to improve robustness to common corruptions, which is achieved by leveraging a diverse set of basic data augmentations (e.g., translate, posterize, solarize) and applying consistency regularization (using the Jensen-Shannon divergence loss). This approach improves robustness when compared to the previously discussed data augmentation methods. Recently, \cite{wang2021augmax} proposed AugMax, a strong data augmentation scheme, that learns an adversarial mixture of several sampled augmentation operators.

Noise injections are another form of data augmentations, which introduce regularization into the training process~\cite{bishop1995training}. Most commonly, the models are trained  with noise perturbed inputs~\cite{an1996effects}. Noise can also be injected into the activation functions~\cite{gulcehre2016noisy}, or the hidden layers of a neural network~\cite{camuto2020explicit,lim2021noisy}.
Recently, \cite{rusak2020simple} has demonstrated that training a model with various noise types helps to improve robustness to unseen common corruptions.

Another recent approach is based on model compression~\cite{diffenderfer2021a}. In this work, the authors show that `lottery ticket-style' pruning methods applied to overparameterized models can substantially improve model robustness to common corruptions.

\section{Method}\label{sec:method}

\subsection{Formulation of \textit{NoisyMix}}
We consider multi-class classification with $K$ labels. Denote the input space by $\mathcal{X} \subset \RR^d$ and the output space by $\mathcal{Y} := \Delta^{K-1}$, 
where $\Delta^{K-1}$ is the probability simplex, i.e., a categorical distribution over $K$ classes. The  classifier, $g$, is constructed from a learnable deep neural network map $p: \mathcal{X} \to \Delta^{K-1}$, mapping an input $x$ to its label, $g(x) = \arg\max_k p^k(x) \in [K]$. We are given a training set, $\mathcal{Z}_n := \{( x_i, y_i)\}_{i \in [n]} \cup \{( A(x_i), y_i)\}_{i \in [n]} $, consisting of $2n$ pairs of input and one-hot label, with 
each training pair $z_{i} := (x_i, y_i), z_{ami} := (A(x_i), y_i) \in \mathcal{X} \times \mathcal{Y}$.  The training pairs $(x_i, y_i)$ are drawn independently from a ground-truth distribution $\mathcal{D} := \mathcal{D}_x \times \mathcal{D}_y$. Here, $A$  denotes the random perturbation of an input, i.e., the {\it AugmentAndMix} function in  \cite{hendrycks2020augmix}. In the sequel, we denote $\mathcal{A}(x)$  as the underlying conditional (on $x$) distribution from which the $A(x)$ is sampled from.  

More precisely, for an (original) input $x$, we construct a virtual data point $x_{am}$ as
\begin{equation}
    A(x) = m x + (1-m) \sum_{i=i}^3 w_i C(x) \sim \mathcal{A}(x),   
\end{equation}
where $m \sim Beta(\alpha, \alpha)$, the $w_i \sim Dirichlet(\alpha, \dots, \alpha)$ (with $\alpha := 1$), and $C$ uniformly drawn from a set of operations whose elements including compositions of transformations such as rotating, translating, autocontrasting, equalizing, posterizing, solarizing, and shearing. 
They can be viewed as randomly perturbed input data, i.e., $A(x) = x + \delta x$, where the random perturbation  $\delta x := (1-m) \left(\sum_{i=1}^3 w_i C(x) - x  \right)$ is additive and data-dependent.
The transformed data are highly diverse because the operations are sampled stochastically and then layered on the same input image. Such diverse transformations are crucial for enabling model robustness by discouraging the model from memorizing fixed augmentations \cite{hendrycks2020augmix}.

Within this setting, we propose the expected \textit{NoisyMix}  loss:
\begin{equation}
    L^{\textit{NoisyMix} } = L^{NFM} + \gamma L^{stability},
\end{equation}
where $\gamma \geq 0$. This loss is a sum of two components: the expected NFM loss \cite{lim2021noisy} and an expected  stability loss \cite{zheng2016improving}. As in \cite{lim2021noisy}, the expected noisy feature mixup (NFM) loss is given by:
\begin{align*}
    L^{NFM} &= \mathbb{E}_{(x, y), (x', y') \sim \mathcal{D} } 
     \mathbb{E}_{\lambda 
    \sim Beta(\alpha, \beta)} \mathbb{E}_{\vecc{\xi} \sim \mathcal{Q}} [  l(p(M_{\lambda, \vecc{
\xi}} (x, x')), M_{\lambda}(y,y'))],
\end{align*}
for some loss function (e.g., cross-entropy) $l:\Delta^{K-1} \times \Delta^{K-1} \to [0, \infty)$ (note that the dependence of both $l$ and $p$ on the learnable parameter $\theta$ is suppressed in the notation), $\vecc{\xi} := (\xi^{add}, \xi^{mult})$ are drawn from some probability distribution $\mathcal{Q}$ with finite first two moments, and
\begin{align*}
    M_{\lambda, \vecc{
\xi}} (x,x') &= (\mathds{1}+\sigma_{1} \xi^{mult}) \odot M_\lambda(x, x') + \sigma_{2} \xi^{add},
\end{align*} 
where $\sigma_1$ and $\sigma_2$ are tuning parameters.
In other words, this approach combines stability training with mixup and noise injections at the level of input and hidden layers. 
In App.~\ref{sm_demon} we illustrate that the combination of these techniques introduce implicit regularization that makes the decision boundary smoother. In turn, this helps improving robustness when predicting on out of distribution data. 

We choose the distance measure $D$ in the expected stability loss to be the Jensen-Shannon divergence (JSD) and minimize the deviation between the distribution on the clean input data and that on the transformed input data, i.e., we minimize $L^{stability} := L^{JSD}$, given by:
\begin{align*}
    & \mathbb{E}_{x, x' \sim \mathcal{D}_x} \mathbb{E}_{x_{am} \sim \mathcal{A}(x), x_{am}' \sim \mathcal{A}(x')}
     \mathbb{E}_{\lambda 
    \sim Beta(\alpha, \beta)} \mathbb{E}_{\vecc{\xi} \sim \mathcal{Q}} JS_\pi(p(M_{\lambda, \vecc{\xi}} (x, x')), p(M_{\lambda, \vecc{\xi}}(x_{am}, x'_{am}) )),
\end{align*}
where $JS_\pi$ denotes Jensen-Shannon divergence with the weights $\pi$, defined as follows: for $p_1, p_2 \in \Delta^{K-1}$ with the corresponding weights $\pi := (\pi_1, \pi_2) \in \Delta$, the JSD~\cite{lin1991divergence} between $p_1$ and $p_2$ is
\begin{align*}
    JS_\pi(p_1, p_2) &:= H(\pi_1 p_1 + \pi_2 p_2) - \pi_1 H(p_1) - \pi_2 H(p_2),
\end{align*}
with $H$ the Shannon entropy, defined as $H(p) := -\sum_{k=1}^K p_k \ln(p_k)$ for $p \in \Delta^{p-1}$. Alternatively,  $JS_\pi(p_1, p_2) =  \pi_1 KL(p_1 \| \pi_1 p_1 + \pi_2 p_2) + \pi_2 KL(p_2 \| \pi_1 p_1 + \pi_2 p_2)$, where $KL$ denotes the Kullback–Leibler divergence.

JSD measures the similarity between two probability distributions, and its square root is a true metric between distributions. 
Unlike the KL divergence and cross-entropy, it is symmetric, bounded, and does not require absolute continuity of the involved probability distributions. Similar to the KL divergence, $JS_\pi(p_1, p_2) \geq 0$, with equality if and only if $p_1 = p_2$. This follows from applying Jensen's inequality to the concave Shannon entropy. Moreover, JSD can interpolate between cross-entropy and mean absolute error for $\pi_1 \in (0,1)$ (see Proposition 1 in \cite{englesson2021generalized}). In GANs, JSD is used as a measure to quantify the similarity between the generative data distribution and the real data distribution \cite{goodfellow2020generative, weng2019gan}. We refer to \cite{lin1991divergence, englesson2021generalized} for more properties of JSD. Note that, unlike $L^{NFM}$, $L^{JSD}$ is computed without the use of the labels.

\textbf{A Note on Stability Training.}
We consider a stability training scheme on the augmented training data \cite{zheng2016improving}. 
Stability training stabilizes the output of a model against small perturbations to input data, thereby enforcing robustness. 
Since the classifiers are probabilistic in nature, we consider a probabilistic notion of robustness, which amounts to making them locally Lipschitz with respect to some distance on the input and output space.
This ensures that a small perturbation in the input will not lead to large changes (as measured by some statistical distance) in the output.

We now formalize a notion of robustness. 
Let $p > 0$. 
We say that a model $f: \mathcal{X} \to \mathcal{Y}$ is  \emph{$\alpha_p$-robust} if for any $(x,y) \sim \mathcal{D}$ such that $f(x) = y$, one has, for any data perturbation $\delta x \in \mathcal{X}$, 
\begin{equation*}
    \|\delta x\|_p \leq \alpha_p \implies f(x) = f(x+\delta x).
\end{equation*}
An analogous definition can be given for distribution-valued outputs, which leads to a notion of probabilistic robustness. 
Let  $D$ be a metric or divergence between two probability distributions. 
We say that a model $F: \mathcal{X} \to \mathcal{P}(\mathcal{Y)}$, i.e., the space of probability distributions on $\mathcal{Y}$, is \emph{$(\alpha_p, \epsilon)$-robust} with respect to $D$ if, for any $x, \delta x \in \mathcal{X}$, 
\begin{equation*}
    \|\delta x\|_p \leq \alpha_p \implies D(F(x), F(x+\delta x)) \leq \epsilon.
\end{equation*}

In practice, stability training involves formulating a loss that aims to flatten the model output in a small neighborhood of any input data, forcing the output to be similar between the original and perturbed copy. 
This is done by combining the cross-entropy loss with a suitable stability objective, encouraging model prediction to be more constant around the input data while mitigating underfitting.

\subsection{Theoretical Results}
\textit{NoisyMix}  seeks to minimize a stochastic approximation of $L^{\textit{NoisyMix} }$ by sampling a finite number of $\lambda, \vecc{\xi}, x_{am}, x'_{am}$ values and using minibatch gradient descent to minimize this loss approximation. The empirical loss to be minimized during the training is 
\begin{equation*}
    L_n^{\textit{NoisyMix} } = \frac{1}{n^2} \sum_{i=1}^n \sum_{j=1}^n \mathbb{E}_{\lambda 
    \sim Beta(\alpha, \beta)} \mathbb{E}_{\vecc{\xi} \sim \mathcal{Q}}[a_{i,j} + \gamma b_{i,j}],    
\end{equation*} 
where
\begin{align*}
    a_{i,j} &:=   l(p(M_{\lambda, \vecc{
\xi}} (x_i, x_j)), M_{\lambda}(y_i,y_j)), \\
    b_{i,j} &:= 
   JS_\pi(p(M_{\lambda, \vecc{
    \xi}} (x_i, x_j)), p(M_{\lambda, \vecc{
\xi}}(A(x_{i}), A(x_{j})) )).
\end{align*}

In App.~\ref{sm_demon}, we illustrate the effectiveness of {\it NoisyMix} in smoothing the decision boundary for a binary classification task. Overall, we observe that {\it NoisyMix} leads to the smoothest decision boundary when compared to other methods (see Fig. \ref{fig:toy_data}), and thus more robust models, since the predicted label stays the same even if the data are perturbed. 

To understand these regularizing effects of {\it NoisyMix} better, we provide some mathematical analysis via the lens of implicit regularization \cite{Mah12} in App.~\ref{sm_secA} (see Theorem \ref{thm_implicitreg} and Theorem \ref{thm_jsd}). Such lens was adopted to understand stochastic optimization \cite{ali2020implicit, smith2021origin} and regularizing effects of noise injections \cite{ camuto2020explicit,lim2021noisyRNN,gong2020maxup}. One common approach to study implicit regularization is to approximate it by an appropriate explicit regularizer, which we follow here for the regime of small mixing coefficient and noise levels.
In this regime, we shall  show that \textit{NoisyMix}  can lead to models with enhanced robustness and stability. 

Following \cite{lamb2019interpolated, lim2021noisy}, we consider the binary cross-entropy loss, setting $h(z) = \log(1+e^z)$, with the labels $y$ taking value in $\{0,1\}$ and the  classifier model $p: \RR^d \to \RR$. In the following, we assume that the model parameter  $\theta \in \Theta := \{\theta : y_i p(x_i) + (y_i - 1) p(x_i) \geq 0 \text{ for all } i \in [n] \}$. We remark that this set contains the set of all parameters with correct classifications of training samples (before applying \textit{NoisyMix}), since $ \{\theta : 1_{\{p(x_i)  \geq 0\}} = y_i \text{ for all } i \in [n] \} \subset \Theta$. Therefore, the condition of $\theta \in \Theta$ is fulfilled when the model classifies all labels correctly for the training data before applying \textit{NoisyMix}. Since the training error often becomes zero in finite time in practice, we shall study the effect of \textit{NoisyMix}  on model robustness in the regime of $\theta \in \Theta$. 

Working in the data-dependent parameter space $\Theta$, we  obtain the following result, which is an informal version of our main theoretical result (see Theorem \ref{thm_advbound}).

\begin{theorem*}[Informal] Let $\epsilon > 0$ be a small parameter, and rescale $1-\lambda \mapsto \epsilon (1-\lambda)$, $\sigma_{add} \mapsto \epsilon \sigma_{add}$, $\sigma_{mult} \mapsto \epsilon \sigma_{mult}$. Then, under additional reasonable assumptions, 
\begin{align}
    L_n^{\textit{NoisyMix} } &\geq \frac{1}{n} \sum_{i=1}^n \max_{\|\delta_i \|_2 \leq \epsilon_i^{mix}} l(p(x_i + \delta_i), y_i)  + \gamma \tilde{L}_{n}^{JSD} +  \epsilon^2 L_n^{reg} +   \epsilon^2  \phi(\epsilon),
\end{align}
for some data-dependent radii $\epsilon_i^{mix} > 0$, stability objective $\tilde{L}_n^{JSD}$, data-dependent regularizer $L_n^{reg}$, and  function $\phi$ such that $\lim_{\epsilon \to 0} \phi(\epsilon) = 0$. 
\end{theorem*}

This theorem implies that  minimizing the \textit{NoisyMix} loss results in a small {\it regularized} adversarial loss and a  stable model. 
Therefore, training with \textit{NoisyMix} not only enhances robustness with respect to input perturbations, but it also imposes additional smoothness. Lastly, we remark that one could generalize the JSD (designed for only two distributions) to provide a similarity measure for any finite number of distributions. This leads to the following definition \cite{lin1991divergence}:
\begin{align*}
    JS_\pi(p_1, \dots, p_K) 
    &:= H\left(\sum_{k=1}^K \pi_k p_k \right) - \sum_{k=1}^K \pi_k H(p_k) 
    = \sum_{k=1}^K \pi_k KL\bigg(p_k \bigg\|  \sum_{j=1}^K \pi_j p_j \bigg),
\end{align*}
where $p_1, \dots, p_K$ are $K$ probability distributions with weights $\pi := (\pi_1, \dots, \pi_K)$.  We note that such a measure, with $K:=3$ and the $\pi_i := 1/3$, will be used in our experiments and was also used in \cite{hendrycks2020augmix} for obtaining optimal results.

\section{Experimental Results}\label{sec:results}

\subsection{Experiment Details}

\paragraph{Datasets.} 
We center our experiments around several recently introduced datasets to benchmark the robustness of neural networks to common corruptions and perturbations. 

The ImageNet-C, CIFAR-10-C, and CIFAR-100-C  datasets \cite{hendrycks2018benchmarking} provide various corrupted versions of the original ImageNet~\cite{deng2009imagenet} and CIFAR~\cite{krizhevsky2009learning} test sets. Specifically, each dataset emulates noise, blur, weather, and digital distortions. In total, there are 15 different perturbation and corruption types, and for each type, there are sets with 5 severity levels which allows us to study the performance with respect to increasing data shifts. Figure~\ref{fig:cimages} in the App. illustrates the different corruption types. The average accuracy across all corruption types and severity levels provides a score for robustness. 

We also consider the ImageNet-R dataset~\cite{hendrycks2021many}, which is designed to measure the robustness of models to various abstract visual renditions (e.g., graffiti, origami, cartoons). 
This dataset provides 30,000 image renditions for a subset of 200 ImageNet object classes. 
The collected renditions have textures and local image statistics that are distinct from the standard ImageNet examples. 
This dataset provides a test for robustness to naturally occurring domain shifts, and it is complementary to ImageNet-C, which is used to test robustness to synthetic domain shifts.

In addition, we consider the ImageNet-P dataset~\cite{hendrycks2018benchmarking}, which provides 10 different perturbed sequences for each ImageNet validation image. 
Each sequence contains 30 or more frames, where each subsequent frame is a perturbed version of the previous frame. 
The applied perturbations are relatively subtle, ensuring that the frames do not go too far out-of-distribution. 
This dataset is specifically designed to evaluate the prediction stability of a model $f$, using the flip probability as a metric to evaluate stability.
Specifically, the flip probability is computed as
$ FB_p=\frac{1}{m(n-1)} \sum_{i=1}^{m}\sum_{j=2}^{n}  I(f(x_j^{(i)}) \ne f(x_{j-1}^{(i)})  ),$
for given $m$ sequences $\{(x_i^{(i)}, \dots, x_n^{(i)})\}_{i=1}^m$ of length $n$ for a corruption type $p$.
Here $I$ denotes an indicator function. 
Unstable models are characterized by a high flip probability, i.e., the prediction behavior across consecutive frames is erratic, whereas consistent predictions indicate stable statistical predictions. 
For sequences that are constructed with noise perturbations (i.e., white, shot and impulse noise), the following modified flip probability is computed:
$ FB_p=\frac{1}{m(n-1)} \sum_{i=1}^{m}\sum_{j=2}^{n}  I(f(x_j^{(i)}) \ne f(x_{i}^{(i)})),$
which accounts for the fact that these sequences are not temporally related. 
We obtain a flip rate by dividing the flip probability of a given model with the  flip probability of AlexNet~\cite{hendrycks2018benchmarking}.  
Finally, the flip rate is averaged across all corruption types  to provide an overall score, denoted as mean flip rate (mFP).

\paragraph{Baselines and Training Details.} 
We consider several data augmentation schemes as baselines, including style transfer~\cite{geirhos2018imagenet}, AutoAugment~\cite{cubuk2019autoaugment}, Mixup~\cite{zhang2018mixup}, Manifold Mixup~\cite{verma2019manifold}, CutMix~\cite{yun2019cutmix}, PuzzleMix~\cite{kim2020puzzle} and AugMix~\cite{hendrycks2020augmix}. 
In our ImageNet experiments, we consider the standard ResNet-50~\cite{he2016deep} architecture as a backbone, trained on ImageNet-1k. We use publicly available models, and for a fair comparison we do not consider models that have been pretrained on larger datasets or that have been trained with any extra data to achieve better performance.
In our CIFAR-10 and CIFAR-100 experiments, we use preactivated ResNet-18~\cite{he2016identity} and Wide-ResNet-28x2~\cite{zagoruyko2016wide}. We train all models from scratch for 200 epochs, using the same basic tuning parameters. For the augmentation schemes, we use the prescribed parameters in the corresponding papers.

\begin{table*}[!t]
	\caption{Clean test accuracy of ResNet-50 models on ImageNet, and average robust accuracy on ImageNet-C and ImageNet-R (higher values are better). In addition, we show the mean flip rate for ImageNet-P (lower values are better) and the robustness to adversarial examples constructed with FGSM (higher values are better). The values in parenthesis are the average robust accuracy for ImageNet-C and mean flip probability for ImageNet-P excluding noise perturbations.}
	\label{tab:results_imagenet}
	\centering
	\scalebox{0.9}{
		\begin{tabular}{lccccccccc}
			\toprule
			& ImageNet ($\uparrow$\%) & ImageNet-C ($\uparrow$\%) & ImageNet-R ($\uparrow$\%) & ImageNet-P ($\downarrow$\%) \\			
			\midrule 
			Baseline~\cite{he2016deep}								& 76.1 			& 39.2 (42.3) 		& 36.2		& 58.0 (57.8) \\
			Adversarial Trained~\cite{NEURIPS2020_24357dd0} & 63.9 			& 32.1 (35.4)		& 38.9		& 33.3 (33.4)  \\
			Stylized ImageNet~\cite{geirhos2018imagenet} 	& 74.9 			& 45.2 (46.6)		& 41.5		& 54.4 (55.2) \\
			AutoAugment~\cite{cubuk2019autoaugment}			& 77.6 			&45.7 (47.3)	& 39.0		& 56.5 (57.7) 	\\
			
			Mixup~\cite{zhang2018mixup} 					& 77.5			& 46.2 (48.4)		& 39.6		& 56.4 (58.7) \\
			Manifold Mixup~\cite{verma2019manifold} 		& 76.7 			& 43.9 (46.5)		& 39.7		& 56.0 (58.2) \\
			CutMix~\cite{yun2019cutmix}						& 78.6			& 41.0 (43.1)		& 34.8		& 58.6 (59.9)  \\
			Puzzle Mix~\cite{kim2020puzzle}					& \textbf{78.7} & 44.6 (46.4)	& 39.5		& 55.5	(57.0)	\\
			AugMix~\cite{hendrycks2020augmix} 				& 77.5			&   48.3 (50.5) 	& 41.0			& 37.6 (37.2) \\

			\textit{NoisyMix}  (ours) 									& 77.6 & \textbf{52.3 (52.4)}  & \textbf{45.7}  & \textbf{28.5 (29.7)} \\			
			
			\bottomrule
	\end{tabular}}
\end{table*}

\subsection{ImageNet Results}

Table~\ref{tab:results_imagenet} summarizes the results for ImageNet models, trained with different data augmentation schemes. Our \textit{NoisyMix} scheme leads to a model that is substantially more robust, and it also improves the test accuracy on clean data as compared to the baseline model. 
%
In contrast, adversarial training and style transfer reduce the in-distribution test accuracy.

\begin{wrapfigure}{r}{0.5\textwidth}
\vspace{-0.2cm}
	\centering
	\includegraphics[width=0.5\textwidth]{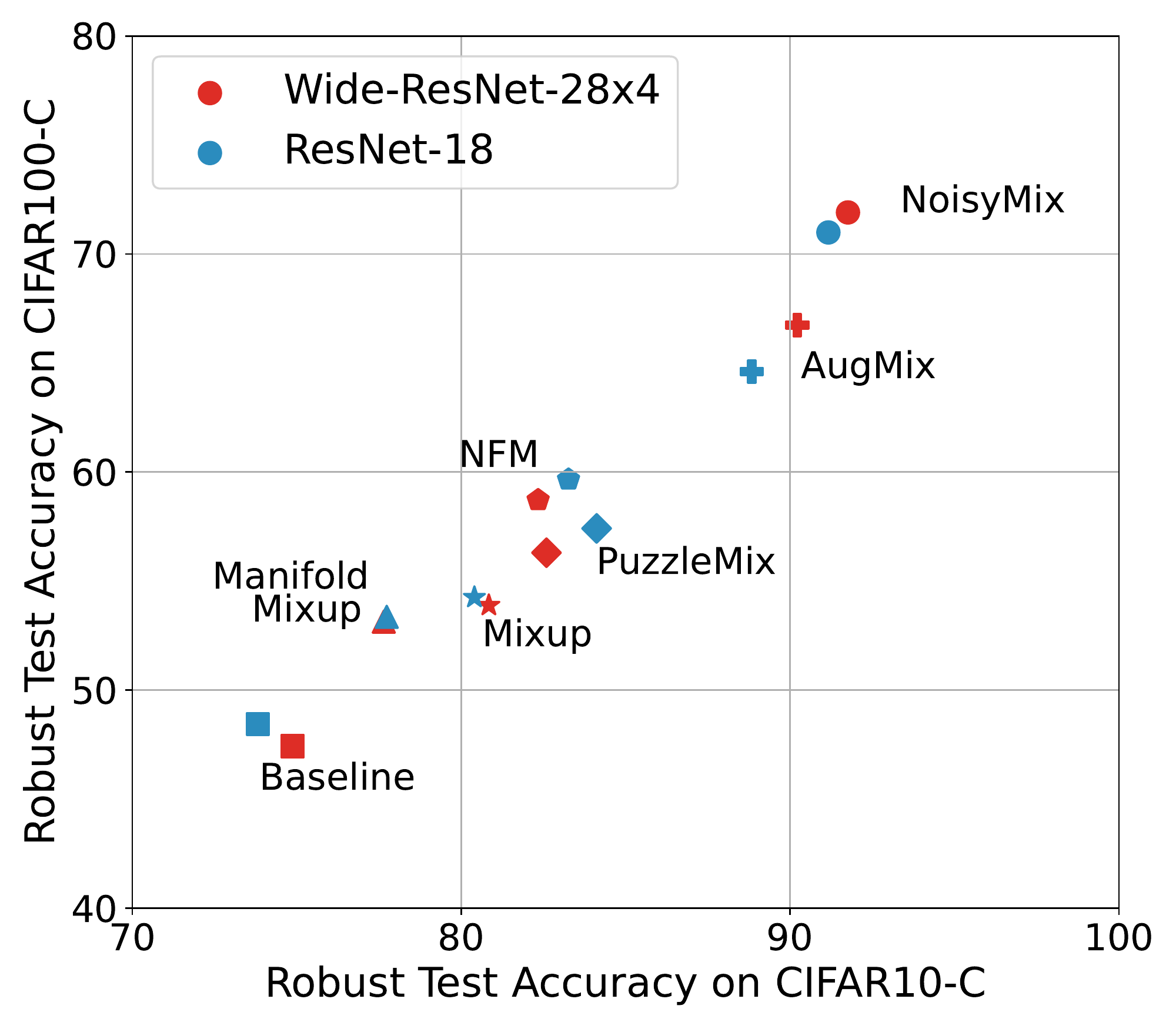}
	\vspace{-0.4cm}
	\caption{Robust accuracy on CIFAR-10-C and CIFAR-100-C for models trained with different data augmentation schemes. \textit{NoisyMix} shows the best performance.}
	\label{fig:robust10vs100}
\end{wrapfigure}
On ImageNet-C, \textit{NoisyMix} gains about $13\%$ compared to the baseline model and about $4\%$ compared to AugMix. Since \textit{NoisyMix} uses white noise during training, we also show the average robust accuracy in parentheses, excluding any noise perturbations (i.e., excluding white, shot and impulse noise). Table~\ref{tab:results_imagenet_c} in the App. provides  results for each perturbation type.

The advantage of \textit{NoisyMix} is also pronounced when evaluated on real-world examples as provided by the ImageNet-R dataset. This shows that \textit{NoisyMix}  not only improves the robustness with respect to synthetic distribution shifts but also to natural distribution shifts. It can also be seen that the stylized ImageNet model performs well on this task.
We also use the ImageNet-R dataset to show that \textit{NoisyMix} yields well-calibrated predictions. Table~\ref{tab:results_imagenet_r_calibration} in the App. shows the RMS calibration error and the area under the response rate accuracy curve (AURRA).

\textit{NoisyMix} also achieves state-of-the-art result on ImageNet-P, reducing the mFP from 58\% down to 28.5\%. Again, we also show the mFP computed for a subset that excludes noise perturbations (i.e., white and shot noise). Table~\ref{tab:results_imagenet_c} in the App. shows detailed results for all perturbation types.

\subsection{CIFAR-10 and CIFAR-100 Results}

We train ResNet-18 and Wide-ResNet-28x2 models (with 5 seeds) on both CIFAR-10 and CIFAR-100, using different data augmentation schemes. 
Figure~\ref{fig:robust10vs100} shows the robust accuracy for the models evaluated on CIFAR-10-C and CIFAR-100-C. 
Here AugMix and \textit{NoisyMix} show the best performance, and both schemes benefit from the use of wide architectures.
In contrast, wide architectures hurt robustness when models are trained with PuzzleMix, and NFM. 
These results show that Wide-ResNets are not always better than ResNets, which is in agreement with results by \cite{bello2021revisiting}. 

Table~\ref{tab:results_cifar10} and~\ref{tab:results_cifar100} provide detailed results, showing the average, minimum, and maximum test accuracy for the various test settings. Similar to the previous ImageNet results, PuzzleMix achieves the best test accuracy on clean test sets. Further, we can see that the advantage of \textit{NoisyMix} is pronounced on CIFAR-100-C, where we improve the robust accuracy by more than $5\%$ as compared to AugMix, and this is independent of the specific architecture. 
\begin{table*}[!t]
	\caption{Results for ResNet-18 models trained on CIFAR-10/100 and evaluated on CIFAR-10/100-C. We train all models with 5 different seeds and show the mean and (minimum/maximum) test accuracy. (Note, the results for AugMax, indicated by a $^*$, are adapted from~\cite{wang2021augmax}.) }
	\label{tab:results_cifar10}
	\centering
	\scalebox{0.9}{
		\begin{tabular}{lcc | ccccccc}
			\toprule
			& CIFAR-10 ($\uparrow$\%) & CIFAR-10-C ($\uparrow$\%) & CIFAR-100 ($\uparrow$\%) & CIFAR-100-C ($\uparrow$\%) \\			
			\midrule 
			Baseline~\cite{he2016identity} 			& 95.3 (95.2 / 95.4)	     & 73.8 (73.5 / 74.2)  	& 77.7 (77.5 / 78.1) & 48.4 (48.0 / 48.8)\\
			Mixup~\cite{zhang2018mixup} 			& 95.8 (95.3 / 96.1)	     & 80.4 (79.7 / 80.9) 	& 79.7 (79.5 / 79.9) & 54.2 (53.9 / 54.8) \\
			Manifold Mixup~\cite{verma2019manifold} & 95.9 (95.7 / 96.1)		  & 77.7 (77.3 / 78.2)	& \textbf{80.3} (\textbf{79.8} / \textbf{80.7}) & 53.4 (52.9 / 53.6) \\
			CutMix~\cite{yun2019cutmix} 			& \textbf{96.5} (\textbf{96.4} / \textbf{96.6}) & 73.0 (71.8 / 73.9)	& 79.3 (78.8 / 80.0) & 47.9 (47.6 / 48.3)\\
			Puzzle Mix~\cite{kim2020puzzle} 			& \textbf{96.5} (96.3 / \textbf{96.6}) & 84.1 (83.7 / 84.3)	& 79.4 (79.1 / 79.6) & 57.4 (56.9 / 57.8)\\
			NFM~\cite{lim2021noisy} 				& 95.4 (95.2 / 95.6)		  & 83.3 (82.7 / 83.9)	& 79.4 (78.9 / 79.8) & 59.7 (59.2 / 60.2)\\
			AugMix~\cite{hendrycks2020augmix} 		& 95.5 (95.0 / 95.8)		& 88.8 (88.1 / 89.3) 	& 77.0 (76.6 / 77.3) & 64.6 (64.4 / 64.8) \\
			AugMax$^*$~\cite{wang2021augmax} 		& -		& 90.36 	& - & 65.75  \\
			\textit{NoisyMix}  (ours) 						& 95.3 (95.2 / 95.3)	       & \textbf{91.2} (\textbf{91.1} / \textbf{91.3})	& 79.7 (79.6 / 79.9) & \textbf{71.0} (\textbf{70.3} / \textbf{71.3})\\			
			\bottomrule
	\end{tabular}}
\end{table*}

\begin{table*}[!t]
	\caption{Results for Wide-ResNet-28x2 trained on CIFAR-10/100 and evaluated on CIFAR-10/100-C. We train all models with 5 different seeds and show the mean and (minimum/maximum) test accuracy.}
	\label{tab:results_cifar100}
	\centering
	\scalebox{0.9}{
		\begin{tabular}{lccccccccc}
			\toprule
			& CIFAR-10 ($\uparrow$\%) & CIFAR-10-C ($\uparrow$\%) & CIFAR-100 ($\uparrow$\%) & CIFAR-100-C ($\uparrow$\%) \\			
			\midrule 
			Baseline~\cite{zagoruyko2016wide} 		& 96.0 (95.9 / 96.1)				& 74.9 (74.5 / 75.4)		& 79.0 (78.8 / 79.1)			& 47.4 (47.0 / 48.0)\\
			Mixup~\cite{zhang2018mixup} 			& 96.7 (96.6 / 96.8)				& 80.8 (80.7 / 81.0)		& 80.9 (80.8 / 81.0)			& 53.9 (53.1 / 54.3)\\
			Manifold Mixup~\cite{verma2019manifold} & 96.6 (96.5 / 96.7)				& 77.7 (76.3 / 78.4)		& \textbf{81.8} (\textbf{81.5} / \textbf{82.0}) 			& 53.2 (52.4 / 53.7) \\
			
			CutMix~\cite{yun2019cutmix} 			& 96.7 (96.6 / 96.8)				& 71.6 (71.1 / 71.9) 		& 80.3 (80.0 / 80.6)			& 48.5 (48.1 / 48.9)\\
			
			Puzzle Mix~\cite{kim2020puzzle} 			& \textbf{97.0} (\textbf{96.9} / \textbf{97.2})  		& 82.6 (82.1 / 83.0)		& 81.2 (81.0 / 81.6)  & 56.3 (55.8 / 56.7) \\
			NFM~\cite{lim2021noisy} 				& 96.2 (96.1 / 96.3)				& 82.3 (81.8 / 82.8)			& 80.2 (79.6 / 80.6)  & 58.7 (58.2 / 59.2)\\
			
			AugMix~\cite{hendrycks2020augmix} 		& 96.5 (96.4 / 96.6)				& 90.2 (89.7 / 90.7)			& 80.5 (79.9 / 80.8) & 66.7 (65.9 / 67.2) \\
			\textit{NoisyMix}  (ours) 				& 96.3 (96.2 / 96.4)			& \textbf{91.8} (\textbf{91.2} / \textbf{92.7}) 	& 81.3 (81.0 / 81.6)	& \textbf{71.9} (\textbf{71.7} / \textbf{72.3}) \\
			\bottomrule
	\end{tabular}}
\end{table*} 

\subsection{Ablation Study}

At a high-level, we can decompose \textit{NoisyMix} into 3 building blocks: (i) Noisy Feature Mixup, (ii) stability training, and (iii) data augmentations. Further, Noisy Feature Mixup can be decomposed into feature mixup and noise injections. We provide a detailed ablation study to show that all components are required to achieve state-of-the-art performance. 

We train models on CIFAR-100 and evaluate the in-domain test accuracy and robust accuracy on CIFAR-100-C. Table~\ref{tab:ablation} shows the results, and it can be seen that data augmentations have a significant impact on robustness, i.e., robust accuracy is improved by 17\%. The combination of feature mixup and noise injections helps to improve the in-domain test accuracy, while the robust accuracy is improved by about 11\%. The combination of both considerably helps to improve in-domain test accuracy and robustness. The performance is further improved when stability training is used in addition, yielding our NosiyMix scheme. Indeed, \textit{NoisyMix} (i.e., the favorable combination of all components) is outperforming all other variants.

\subsection{Does Robustness Transfer?}

Motivated by \cite{shafahi2019adversarially}, who showed that adversarial robustness transfers, we study whether robustness to common corruptions transfers from a source model to a target task. 
To do so, we fine-tune ImageNet models on CIFAR-10 and CIFAR-100, and then use the corrupted test datasets to evaluate robustness. 

Table~\ref{tab:results_transfer} shows that target models inherit some robustness from the source models. More robust ImageNet models, i.e., models trained with \textit{NoisyMix} or Augmix, lead to more robust target models. In particular, robustness to noise perturbations is transferred. For example, the parent model trained with \textit{NoisyMix} yields target models that are about 12\% more robust to noise perturbations on CIFAR-10-C and about 6\% more robust on CIFAR-100-C. Moreover, a robust parent model can help to improve the in-domain accuracy on the target task, which is in alignment with the results by \cite{utrera2020adversarially,NEURIPS2020_24357dd0}.

\begin{table*}[!t]
	\caption{Ablation study using a Wide-ResNet-28x2 trained on CIFAR-100. The combination of feature mixing and noise injections on top of a stability training scheme on an augmented data set boost both clean and robust accuracy. }
	\label{tab:ablation}
	\centering
	\scalebox{0.9}{
		\begin{tabular}{cccc c c  c  c c  c  c  c  c c c c c c c}
			\toprule
			\rotatebox[origin=c]{55}{Augmentations} & \rotatebox[origin=c]{55}{Feature Mixing} & \rotatebox[origin=c]{55}{Noise Injections} & \rotatebox[origin=c]{55}{JSD Loss} & \rotatebox[origin=c]{55}{CIFAR-100 ($\uparrow$\%)} & \rotatebox[origin=c]{55}{CIFAR-100-C ($\uparrow$\%)} &  \rotatebox[origin=c]{55}{Robustness Gain (\%)} \\			
			\midrule 
			\xmark & \xmark  & \xmark & \xmark & 79.0 & 47.4 & - \\
			\midrule
			\xmark & \xmark  & \cmark & \xmark & 79.1 & 49.5  & +2.1 \\	
			\xmark & \cmark  & \xmark & \xmark & 81.8 & 53.2  & +5.8 \\					
			\xmark & \cmark  & \cmark & \xmark & 80.2 & 58.7  & +11.3 \\
			\cmark & \xmark  & \xmark & \xmark & 79.1 & 64.0 & +16.6\\
			\cmark & \xmark  & \cmark & \cmark & 80.9 & 66.7  & +19.3 \\
			\cmark & \xmark  & \xmark & \cmark & 80.5 & 66.9 & +19.5\\
			\cmark & \cmark  & \cmark & \xmark & 80.4 & 68.6 & +21.2\\
			\cmark & \cmark  & \xmark & \cmark & \textbf{82.1} & 69.3 & +21.9\\
			\midrule
			\cmark & \cmark  & \cmark & \cmark & 81.3 & \textbf{71.9} & +\textbf{24.4} \\
			\bottomrule
	\end{tabular}}
\end{table*} 

\begin{table*}[!t]
	\caption{Transfer learning results for ResNet-50 models trained on ImageNet and transfered to CIFAR-10/100. Then robustness is evaluated on CIFAR-10-C/100-C. Specifically, we show the average robust accuracy across all common corruptions / across noise corruptions / across whether, blur and digital corruptions. (Higher values are better.)}
	\label{tab:results_transfer}
	\centering
	\scalebox{0.9}{
		\begin{tabular}{lcc | ccccccc}
			\toprule
			Source Model & CIFAR-10 ($\uparrow$\%) & CIFAR-10-C ($\uparrow$\%) & CIFAR-100 ($\uparrow$\%) & CIFAR-100-C ($\uparrow$\%) \\			
			\midrule 
			Baseline~\cite{he2016deep} 			& 97.0 			& 77.1 / 48.8 / 86.5  			& 84.3 			& 57.1 / 28.0 / 66.8 \\
			
			PuzzleMix~\cite{kim2020puzzle}		& 97.1 			& 77.3 / 47.7 / 87.2 			& 84.5 & 57.2 / 26.1 / 67.6  \\
			
			Mixup~\cite{zhang2018mixup} & 97.1 &  78.6 / 51.7 /	87.6 & 84.5 & 57.3 / 26.4 / 67.6 	\\
			
			AutoAugment~\cite{cubuk2019autoaugment} & 97.2 & 78.7 / 54.6 /	86.7 & 84.9 & 57.6 / 27.3 / 67.6	\\
			
			AugMix~\cite{hendrycks2020augmix} & \textbf{97.7}			& 79.9 / 53.9 / 88.6			& \textbf{86.5}	& 58.8 / 24.2 / 66.8 \\ 
			\textit{NoisyMix}  (ours) & \textbf{97.7} 	& \textbf{81.8 / 60.3 / 88.9}  	& 85.9 			& \textbf{60.3 / 34.5 / 68.8}\\			
			\bottomrule
	\end{tabular}}
\end{table*} 

\section{Conclusion}

Desirable properties of deep classifiers for computer vision are (i) in-domain accuracy, (ii) robustness to domain shifts, and (iii) well-calibrated estimation of class membership probabilities.
Motivated by the challenge to achieve these, we propose \textit{NoisyMix}, a novel and theoretically justified training scheme. Our empirical results demonstrate the advantage, i.e., improvement of (i), (ii), and (iii), when compared to standard training and other recently proposed data augmentation methods. These findings are supported by theoretical results that show that \textit{NoisyMix} can improve model robustness. 

A limitation of \textit{NoisyMix} is that it is tailored towards computer vision tasks and not directly applicable to natural language processing tasks. %
Future work will explore how \textit{NoisyMix} can be extended to a broader range of tasks. In particular, we anticipate that \textit{NoisyMix} will be effective in improving the robustness of models used in Scientific Machine Learning (e.g., models used for climate predictions).
Since this paper studies a new training scheme to improve model robustness there are no potential negative societal impacts of our work. 

\section*{Acknowledgements}
N. B. Erichson and M. W. Mahoney would like to acknowledge IARPA (contract W911NF20C0035), NSF, and ONR for providing partial support of this work. S. H. Lim would like to acknowledge the WINQ Fellowship and the Knut and Alice Wallenberg Foundation for providing support of this work. Our conclusions do not necessarily reflect the position or the policy of our sponsors, and no official endorsement should be inferred. We are also grateful for the generous support from Amazon AWS.

\bibliographystyle{plain}

\bibliography{noisymix}

\begin{thebibliography}{10}

\bibitem{ali2020implicit}
Alnur Ali, Edgar Dobriban, and Ryan Tibshirani.
\newblock The implicit regularization of stochastic gradient flow for least
  squares.
\newblock In {\em Proceedings of the International Conference on Machine
  Learning}, pages 233--244. PMLR, 2020.

\bibitem{an1996effects}
Guozhong An.
\newblock The effects of adding noise during backpropagation training on a
  generalization performance.
\newblock {\em Neural Computation}, 8(3):643--674, 1996.

\bibitem{bello2021revisiting}
Irwan Bello, William Fedus, Xianzhi Du, Ekin~D Cubuk, Aravind Srinivas,
  Tsung-Yi Lin, Jonathon Shlens, and Barret Zoph.
\newblock Revisiting resnets: Improved training and scaling strategies.
\newblock {\em arXiv preprint arXiv:2103.07579}, 2021.

\bibitem{bishop1995training}
Chris~M Bishop.
\newblock Training with noise is equivalent to tikhonov regularization.
\newblock {\em Neural Computation}, 7(1):108--116, 1995.

\bibitem{camuto2020explicit}
Alexander Camuto, Matthew Willetts, Umut {\c{S}}im{\c{s}}ekli, Stephen Roberts,
  and Chris Holmes.
\newblock Explicit regularisation in gaussian noise injections.
\newblock {\em arXiv preprint arXiv:2007.07368}, 2020.

\bibitem{cubuk2019autoaugment}
Ekin~D Cubuk, Barret Zoph, Dandelion Mane, Vijay Vasudevan, and Quoc~V Le.
\newblock Autoaugment: Learning augmentation strategies from data.
\newblock In {\em Proceedings of the Conference on Computer Vision and Pattern
  Recognition}, pages 113--123, 2019.

\bibitem{deng2009imagenet}
Jia Deng, Wei Dong, Richard Socher, Li-Jia Li, Kai Li, and Li~Fei-Fei.
\newblock Imagenet: A large-scale hierarchical image database.
\newblock In {\em Proceedings of the Conference on Computer Vision and Pattern
  Recognition}, pages 248--255, 2009.

\bibitem{devries2017improved}
Terrance DeVries and Graham~W Taylor.
\newblock Improved regularization of convolutional neural networks with cutout.
\newblock {\em arXiv preprint arXiv:1708.04552}, 2017.

\bibitem{diffenderfer2021a}
James Diffenderfer, Brian~R Bartoldson, Shreya Chaganti, Jize Zhang, and Bhavya
  Kailkhura.
\newblock A winning hand: Compressing deep networks can improve
  out-of-distribution robustness.
\newblock In A.~Beygelzimer, Y.~Dauphin, P.~Liang, and J.~Wortman Vaughan,
  editors, {\em Advances in Neural Information Processing Systems}, 2021.

\bibitem{dodge2016understanding}
Samuel Dodge and Lina Karam.
\newblock Understanding how image quality affects deep neural networks.
\newblock In {\em Proceedings of the International Conference on Quality of
  Multimedia Experience}, pages 1--6, 2016.

\bibitem{englesson2021generalized}
Erik Englesson and Hossein Azizpour.
\newblock Generalized {J}ensen-{S}hannon divergence loss for learning with
  noisy labels.
\newblock {\em arXiv preprint arXiv:2105.04522}, 2021.

\bibitem{geirhos2018imagenet}
Robert Geirhos, Patricia Rubisch, Claudio Michaelis, Matthias Bethge, Felix~A
  Wichmann, and Wieland Brendel.
\newblock Imagenet-trained cnns are biased towards texture; increasing shape
  bias improves accuracy and robustness.
\newblock In {\em Proceedings of the International Conference on Learning
  Representations}, 2018.

\bibitem{gong2020maxup}
Chengyue Gong, Tongzheng Ren, Mao Ye, and Qiang Liu.
\newblock Maxup: Lightweight adversarial training with data augmentation
  improves neural network training.
\newblock In {\em Proceedings of the Conference on Computer Vision and Pattern
  Recognition}, pages 2474--2483, 2021.

\bibitem{goodfellow2020generative}
Ian Goodfellow, Jean Pouget-Abadie, Mehdi Mirza, Bing Xu, David Warde-Farley,
  Sherjil Ozair, Aaron Courville, and Yoshua Bengio.
\newblock Generative adversarial networks.
\newblock {\em Communications of the ACM}, 63(11):139--144, 2020.

\bibitem{goodfellow2014explaining}
Ian~J Goodfellow, Jonathon Shlens, and Christian Szegedy.
\newblock Explaining and harnessing adversarial examples.
\newblock {\em arXiv preprint arXiv:1412.6572}, 2014.

\bibitem{gulcehre2016noisy}
Caglar Gulcehre, Marcin Moczulski, Misha Denil, and Yoshua Bengio.
\newblock Noisy activation functions.
\newblock In {\em Proceedings of the International Conference on Machine
  Learning}, pages 3059--3068. PMLR, 2016.

\bibitem{guo2017calibration}
Chuan Guo, Geoff Pleiss, Yu~Sun, and Kilian~Q Weinberger.
\newblock On calibration of modern neural networks.
\newblock In {\em Proceedings of the International Conference on Machine
  Learning}, pages 1321--1330, 2017.

\bibitem{he2016deep}
Kaiming He, Xiangyu Zhang, Shaoqing Ren, and Jian Sun.
\newblock Deep residual learning for image recognition.
\newblock In {\em Proceedings of the Conference on Computer Vision and Pattern
  Recognition}, pages 770--778, 2016.

\bibitem{he2016identity}
Kaiming He, Xiangyu Zhang, Shaoqing Ren, and Jian Sun.
\newblock Identity mappings in deep residual networks.
\newblock In {\em Proceedings of the European Conference on Computer Vision},
  pages 630--645, 2016.

\bibitem{hendrycks2021many}
Dan Hendrycks, Steven Basart, Norman Mu, Saurav Kadavath, Frank Wang, Evan
  Dorundo, Rahul Desai, Tyler Zhu, Samyak Parajuli, Mike Guo, et~al.
\newblock The many faces of robustness: A critical analysis of
  out-of-distribution generalization.
\newblock In {\em Proceedings of the International Conference on Computer
  Vision}, 2021.

\bibitem{hendrycks2018benchmarking}
Dan Hendrycks and Thomas Dietterich.
\newblock Benchmarking neural network robustness to common corruptions and
  perturbations.
\newblock In {\em Proceedings of the International Conference on Learning
  Representations}, 2018.

\bibitem{hendrycks2019using}
Dan Hendrycks, Kimin Lee, and Mantas Mazeika.
\newblock Using pre-training can improve model robustness and uncertainty.
\newblock In {\em Proceedings of the International Conference on International
  Conference on Machine Learning}, pages 2712--2721, 2019.

\bibitem{hendrycks2020augmix}
Dan Hendrycks, Norman Mu, Ekin~D. Cubuk, Barret Zoph, Justin Gilmer, and Balaji
  Lakshminarayanan.
\newblock {AugMix}: A simple data processing method to improve robustness and
  uncertainty.
\newblock {\em Proceedings of the International Conference on Learning
  Representations}, 2020.

\bibitem{hinton2015distilling}
Geoffrey Hinton, Oriol Vinyals, and Jeff Dean.
\newblock Distilling the knowledge in a neural network.
\newblock {\em arXiv preprint arXiv:1503.02531}, 2015.

\bibitem{kim2020puzzle}
Jang-Hyun Kim, Wonho Choo, and Hyun~Oh Song.
\newblock Puzzle mix: Exploiting saliency and local statistics for optimal
  mixup.
\newblock In {\em Proceedings of the International Conference on Machine
  Learning}, 2020.

\bibitem{krizhevsky2009learning}
Alex Krizhevsky.
\newblock Learning multiple layers of features from tiny images.
\newblock 2009.

\bibitem{krizhevsky2012imagenet}
Alex Krizhevsky, Ilya Sutskever, and Geoffrey~E Hinton.
\newblock Imagenet classification with deep convolutional neural networks.
\newblock {\em Proceedings of the International Conference on Neural
  Information Processing Systems}, 25:1097--1105, 2012.

\bibitem{KGC17_TR}
J.~Kuka{\v c}ka, V.~Golkov, and D.~Cremers.
\newblock Regularization for deep learning: A taxonomy.
\newblock Technical Report Preprint: arXiv:1710.10686, 2017.

\bibitem{lamb2019interpolated}
Alex Lamb, Vikas Verma, Juho Kannala, and Yoshua Bengio.
\newblock Interpolated adversarial training: Achieving robust neural networks
  without sacrificing too much accuracy.
\newblock In {\em Proceedings of the 12th ACM Workshop on Artificial
  Intelligence and Security}, pages 95--103, 2019.

\bibitem{lim2021noisyRNN}
Soon~Hoe Lim, N~Benjamin Erichson, Liam Hodgkinson, and Michael~W Mahoney.
\newblock Noisy recurrent neural networks.
\newblock In {\em Proceedings of the International Conference on Neural
  Information Processing Systems}, 2021.

\bibitem{lim2021noisy}
Soon~Hoe Lim, N.~Benjamin Erichson, Francisco Utrera, Winnie Xu, and Michael~W.
  Mahoney.
\newblock Noisy feature mixup.
\newblock In {\em International Conference on Learning Representations}, 2022.

\bibitem{lin1991divergence}
Jianhua Lin.
\newblock Divergence measures based on the {S}hannon entropy.
\newblock {\em IEEE Transactions on Information theory}, 37(1):145--151, 1991.

\bibitem{Mah12}
M.~W. Mahoney.
\newblock Approximate computation and implicit regularization for very
  large-scale data analysis.
\newblock In {\em Proceedings of the 31st ACM Symposium on Principles of
  Database Systems}, pages 143--154, 2012.

\bibitem{rebuffi2021fixing}
Sylvestre-Alvise Rebuffi, Sven Gowal, Dan~A Calian, Florian Stimberg, Olivia
  Wiles, and Timothy Mann.
\newblock Fixing data augmentation to improve adversarial robustness.
\newblock {\em arXiv preprint arXiv:2103.01946}, 2021.

\bibitem{rice2020overfitting}
Leslie Rice, Eric Wong, and Zico Kolter.
\newblock Overfitting in adversarially robust deep learning.
\newblock In {\em Proceedings of the International Conference on Machine
  Learning}, pages 8093--8104. PMLR, 2020.

\bibitem{rusak2020simple}
Evgenia Rusak, Lukas Schott, Roland~S Zimmermann, Julian Bitterwolf, Oliver
  Bringmann, Matthias Bethge, and Wieland Brendel.
\newblock A simple way to make neural networks robust against diverse image
  corruptions.
\newblock In {\em Proceedings of the European Conference on Computer Vision},
  pages 53--69, 2020.

\bibitem{NEURIPS2020_24357dd0}
Hadi Salman, Andrew Ilyas, Logan Engstrom, Ashish Kapoor, and Aleksander Madry.
\newblock Do adversarially robust imagenet models transfer better?
\newblock In {\em Proceedings of the International Conference on Neural
  Information Processing Systems}, volume~33, 2020.

\bibitem{shafahi2019adversarial}
Ali Shafahi, Mahyar Najibi, Amin Ghiasi, Zheng Xu, John Dickerson, Christoph
  Studer, Larry~S Davis, Gavin Taylor, and Tom Goldstein.
\newblock Adversarial training for free!
\newblock In {\em Proceedings of the International Conference on Neural
  Information Processing Systems}, 2019.

\bibitem{shafahi2019adversarially}
Ali Shafahi, Parsa Saadatpanah, Chen Zhu, Amin Ghiasi, Christoph Studer, David
  Jacobs, and Tom Goldstein.
\newblock Adversarially robust transfer learning.
\newblock In {\em Proceedings of the International Conference on Learning
  Representations}, 2019.

\bibitem{smith2021origin}
Samuel~L Smith, Benoit Dherin, David~GT Barrett, and Soham De.
\newblock On the origin of implicit regularization in stochastic gradient
  descent.
\newblock {\em arXiv preprint arXiv:2101.12176}, 2021.

\bibitem{szegedy2013intriguing}
Christian Szegedy, Wojciech Zaremba, Ilya Sutskever, Joan Bruna, Dumitru Erhan,
  Ian Goodfellow, and Rob Fergus.
\newblock Intriguing properties of neural networks.
\newblock {\em arXiv preprint arXiv:1312.6199}, 2013.

\bibitem{tancik2020fourier}
Matthew Tancik, Pratul~P Srinivasan, Ben Mildenhall, Sara Fridovich-Keil,
  Nithin Raghavan, Utkarsh Singhal, Ravi Ramamoorthi, Jonathan~T Barron, and
  Ren Ng.
\newblock Fourier features let networks learn high frequency functions in low
  dimensional domains.
\newblock {\em arXiv preprint arXiv:2006.10739}, 2020.

\bibitem{utrera2020adversarially}
Francisco Utrera, Evan Kravitz, N~Benjamin Erichson, Rajiv Khanna, and
  Michael~W Mahoney.
\newblock Adversarially-trained deep nets transfer better: Illustration on
  image classification.
\newblock In {\em Proceedings of the International Conference on Learning
  Representations}, 2020.

\bibitem{verma2019manifold}
Vikas Verma, Alex Lamb, Christopher Beckham, Amir Najafi, Ioannis Mitliagkas,
  David Lopez-Paz, and Yoshua Bengio.
\newblock Manifold mixup: Better representations by interpolating hidden
  states.
\newblock In {\em Proceedings of the International Conference on Learning
  Representations}, 2019.

\bibitem{wang2021augmax}
Haotao Wang, Chaowei Xiao, Jean Kossaifi, Zhiding Yu, Anima Anandkumar, and
  Zhangyang Wang.
\newblock Augmax: Adversarial composition of random augmentations for robust
  training.
\newblock {\em Advances in Neural Information Processing Systems}, 34, 2021.

\bibitem{weng2019gan}
Lilian Weng.
\newblock From {G}{A}{N} to {W}{G}{A}{N}.
\newblock {\em arXiv preprint arXiv:1904.08994}, 2019.

\bibitem{xie2020self}
Qizhe Xie, Minh-Thang Luong, Eduard Hovy, and Quoc~V Le.
\newblock Self-training with noisy student improves imagenet classification.
\newblock In {\em Proceedings of the Conference on Computer Vision and Pattern
  Recognition}, pages 10687--10698, 2020.

\bibitem{fourier}
Dong Yin, Raphael Gontijo~Lopes, Jon Shlens, Ekin~Dogus Cubuk, and Justin
  Gilmer.
\newblock A fourier perspective on model robustness in computer vision.
\newblock {\em Proceedings of the International Conference on Neural
  Information Processing Systems}, 32:13276--13286, 2019.

\bibitem{yu2013stability}
Bin Yu.
\newblock Stability.
\newblock {\em Bernoulli}, 19(4):1484--1500, 2013.

\bibitem{yun2019cutmix}
Sangdoo Yun, Dongyoon Han, Seong~Joon Oh, Sanghyuk Chun, Junsuk Choe, and
  Youngjoon Yoo.
\newblock Cutmix: Regularization strategy to train strong classifiers with
  localizable features.
\newblock In {\em Proceedings of the International Conference on Computer
  Vision}, 2019.

\bibitem{zadrozny2001obtaining}
Bianca Zadrozny and Charles Elkan.
\newblock Obtaining calibrated probability estimates from decision trees and
  naive bayesian classifiers.
\newblock In {\em Proceedings of the International Conference on Machine
  Learning}, volume~1, pages 609--616, 2001.

\bibitem{zagoruyko2016wide}
Sergey Zagoruyko and Nikos Komodakis.
\newblock Wide residual networks.
\newblock In {\em British Machine Vision Conference 2016}, 2016.

\bibitem{zhang2018mixup}
Hongyi Zhang, Moustapha Cisse, Yann~N Dauphin, and David Lopez-Paz.
\newblock mixup: Beyond empirical risk minimization.
\newblock In {\em Proceedings of the International Conference on Learning
  Representations}, 2018.

\bibitem{zhang2020does}
Linjun Zhang, Zhun Deng, Kenji Kawaguchi, Amirata Ghorbani, and James Zou.
\newblock How does mixup help with robustness and generalization?
\newblock {\em arXiv preprint arXiv:2010.04819}, 2020.

\bibitem{zheng2016improving}
Stephan Zheng, Yang Song, Thomas Leung, and Ian Goodfellow.
\newblock Improving the robustness of deep neural networks via stability
  training.
\newblock In {\em Proceedings of the Conference on Computer Vision and Pattern
  Recognition}, pages 4480--4488, 2016.

\end{thebibliography}

\clearpage
\appendix

\newpage
\appendix
\onecolumn

\begin{center}
    {\Large \bf  Appendix for ``\textit{NoisyMix}"}
\end{center}






In this {\bf Appendix}, we provide additional materials to support our results in the main paper. In particular, we provide  mathematical analysis of \textit{NoisyMix}  to shed light on its regularizing and robustness properties, as well as additional empirical results and the details.

\noindent {\bf Organizational Details.} This {\bf Appendix} is organized as follows.

\begin{itemize}
    \item In Section \ref{sm_demon}, we illustrate the regularizing effects of stability training with the JSD, Manifold Mixup, and noise injection on a toy dataset. 
    \item In Section \ref{sm_secA}, we study \textit{NoisyMix}  via the lens of implicit regularization. Our contributions there are Theorem 1 and Theorem 2, which show that minimizing the \textit{NoisyMix}  loss function is approximately equivalent to minimizing a sum of the original loss and data-dependent regularizers,  amplifying the regularizing effects of AugMix  according to the mixing and noise levels.  
    \item In Section \ref{sm_secD}, we show that \textit{NoisyMix}  training can improve model robustness when compared to standard training. Our main contribution there is Theorem 3, which shows that, under appropriate assumptions, \textit{NoisyMix}  training approximately minimizes an upper bound on the sum of an adversarial loss, a stability objective and data-dependent regularizers.
\item In Section \ref{sm_secG}, we provide additional experimental results and their details.
\end{itemize}

We recall the notation that we use in the main paper as well as this Appendix. 

\noindent {\bf Notation.}
$I$ denotes identity matrix, $[K] := \{1,\dots, K\}$, the superscript $^T$ denotes transposition, $\circ$ denotes composition, $\odot$ denotes Hadamard product, $\mathds{1}$ denotes the vector with all components equal one. For a vector $v$, $v^i$ denotes its $i$th component and $\|v\|_p$ denotes its $l_p$ norm for $p > 0$. $M_\lambda(a,b) := \lambda a + (1-\lambda) b$, for  random variables $a, b, \lambda$. For $\alpha, \beta > 0$,  $\tilde{\mathcal{D}}_{\lambda}:= \frac{\alpha}{\alpha + \beta} Beta(\alpha + 1, \beta) + \frac{\beta}{\alpha + \beta} Beta(\beta + 1, \alpha)$ is a uniform mixture of two Beta distributions. For the two vectors $a,b$, $\cos(a,b) := \langle a, b \rangle/\|a\|_2\|b\|_2$ denotes their cosine similarity.

To simplify mathematical analysis and enable clearer intuition, we shall restrict to the special case when the mixing is done only at the input level.  
Moreover, we consider the setting described in the main paper for the case of $K=2$ and $\pi_0 = \pi_1 = 1/2$ for our  analysis. The results for the general case can be derived analogously using our techniques.


\section{Illustration of the Regularizing Effects of Stability Training, Manifold Mixup and Noise Injection on a Toy Dataset }\label{sm_demon}

We consider a binary classification task for the noise corrupted 2D dataset whose points are separated by a concave polygon-shaped buffer band. Inner and outer points that disjoint the band directly correspond to two label classes, while those falling onto the band are randomly assigned to one of the classes. We generate 500 samples by setting the scale factor to be 0.5 and adding Gaussian noise with zero mean and standard deviation of 0.2 to the shape. Fig.~\ref{fig:toy_data} shows the data points. 


\begin{figure}[!t]
	\centering
	\begin{subfigure}[t]{0.35\textwidth}
		\centering
		\begin{overpic}[width=1\textwidth]{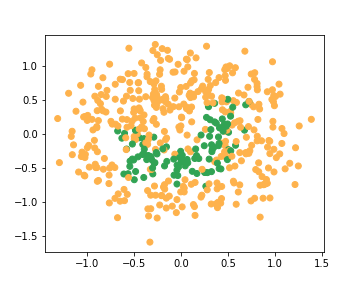}
		\end{overpic}
		\caption{All data points.}
	\end{subfigure}\hspace{+0.5cm}
	~
	\begin{subfigure}[t]{0.35\textwidth}
		\centering
		\begin{overpic}[width=1\textwidth]{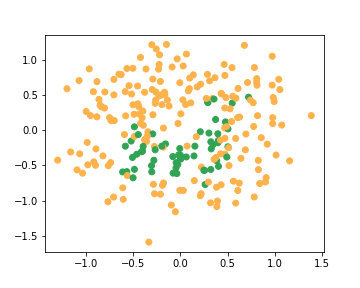} 
		\end{overpic}			
		\caption{Data points used for testing.}
	\end{subfigure}	
	
	\vspace{+0.2cm}
	\caption{The dataset in $\mathbb{R}^2$ that we use for binary classification.}
	\label{fig:toy_data}
\end{figure}

We train a fully connected feedforward shallow neural network (with $32$ neurons) with the ReLU activation function on these data, using 250 points for training and 250 for testing. All models are trained with Adam, learning rate of $0.1$, and batch size of $50$ for 100 epochs. The seed is fixed across all experiments.

We  focus on the effects of stability training (with the JSD), Manifold Mixup, noise injection, and their combinations on the test performance. Note that we do not explore AugMix's data augmentation here, since the operations used there are not meaningful for the toy dataset here.

Fig.~\ref{fig:toy_example} illustrates how different training schemes affect the decision boundaries of the neural network classifier. 
First, we can see that \textit{NoisyMix} (which combines feature mixup, noise injections, and stability training) is most effective at smoothing the decision boundary of the trained classifiers, imposing the strongest smoothness on the dataset. Here it also yields the best test accuracy among all considered training schemes.
%
%
We can also see that stability training helps to boost the effectiveness of using Manifold Mixup.
In contrast, NFM without stability training performs better than the baseline and Manifold Mixup, but not as well as \textit{NoisyMix} and Manifold Mixup with stability training.




\begin{figure}[!t]
	\centering
	\footnotesize
	\stackunder[5pt]{\includegraphics[width=0.3\textwidth]{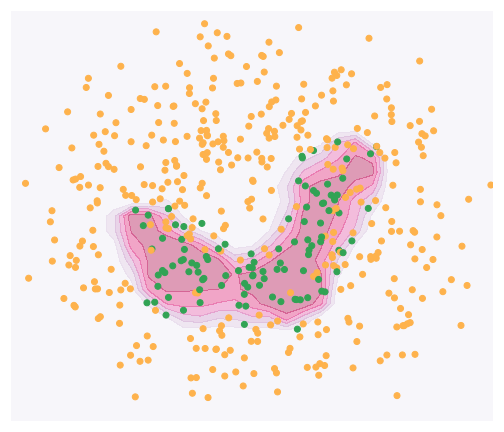}}{Baseline (86.8\%).}
	%
	\stackunder[5pt]{\includegraphics[width=0.3\textwidth]{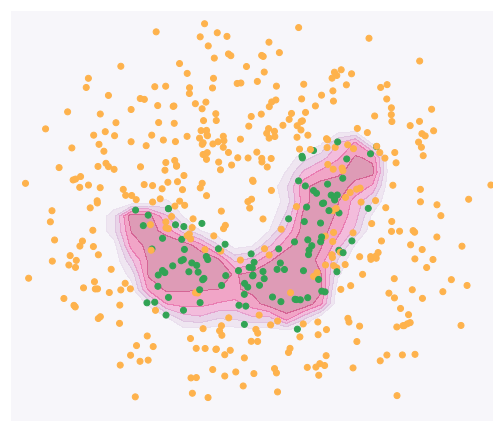}}{Noise injection (86.8\%).} 
	%
	\stackunder[5pt]{\includegraphics[width=0.3\textwidth]{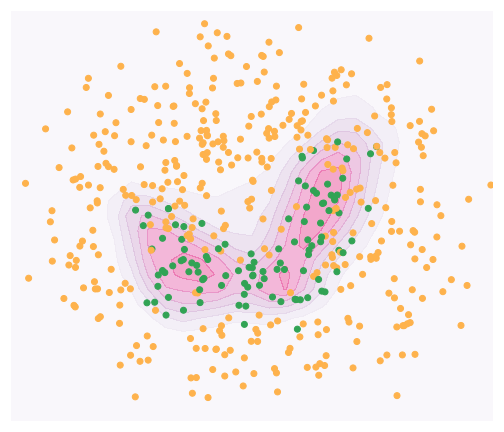}}{Manifold Mixup (87.4\%).}
	\stackunder[5pt]{\includegraphics[width=0.3\textwidth]{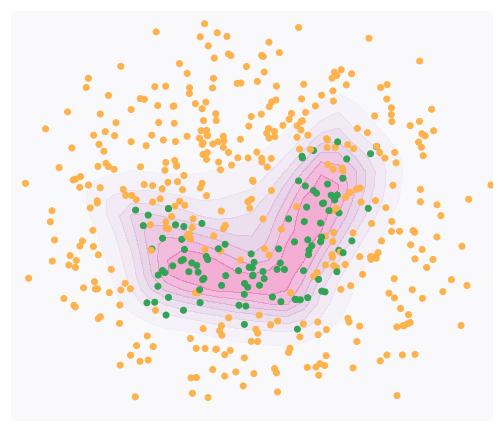}}{NFM (87.6\%).} 
	\stackunder[5pt]{\includegraphics[width=0.3\textwidth]{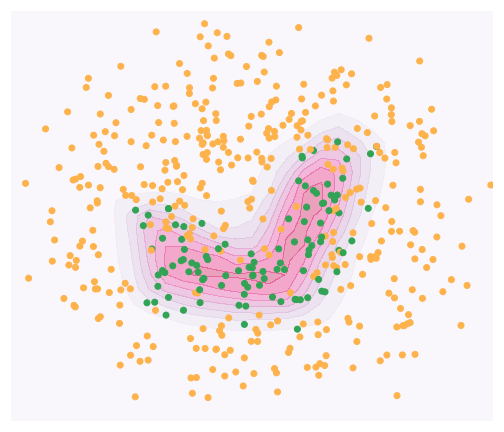}}{Manifold Mixup +  JSD  (88.0 \%).}
	\stackunder[5pt]{\includegraphics[width=0.3\textwidth]{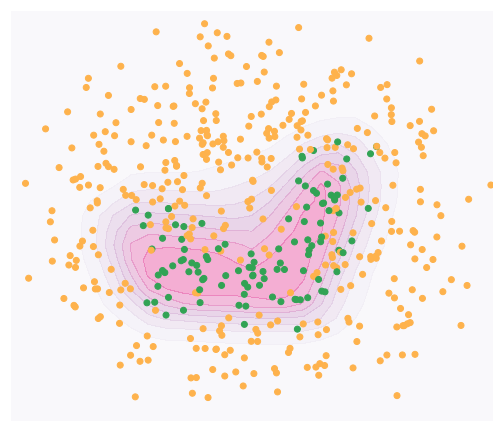}}{\textit{NoisyMix} (88.8\%).} 
	
	\caption{Decision boundaries and test accuracy (in parenthesis) for different training schemes on a 2D binary classification problem. {\it NoisyMix} is most effective at smoothing the decision boundary, and it also yields the best test accuracy among all considered training schemes.}
	\label{fig:toy_example}
\end{figure}

\section{Implicit Regularization of NoisyMix} \label{sm_secA}

In this section, we are going to identify the implicit regularization effects of \textit{NoisyMix}  by analyzing the NFM loss and the stability loss.

We emphasize that while the analysis for the NFM loss  can be adapted from that in \cite{lim2021noisy} in a straightforward manner, the analysis and results for the stability loss (see Theorem \ref{thm_jsd}) are novel and needs to be derived independently. Importantly, Theorem \ref{thm_jsd} is a crucial ingredient for the proof of our main theoretical result, Theorem \ref{thm_advbound} (stated informally in the main paper). The analysis for the stability loss is rather tedious, and we are going to provide the details of the derivation  for completeness (see Subsection \ref{app_subsec:stability}). 
 
\subsection{Implicit Regularization for the NFM Loss}

Recall that \begin{equation}
    L^{NFM} = \mathbb{E}_{(x, y), (x', y') \sim \mathcal{D} } 
     \mathbb{E}_{\lambda 
    \sim Beta(\alpha, \beta)}  \mathbb{E}_{\vecc{\xi} \sim \mathcal{Q}} l(p(M_{\lambda, \vecc{
\xi}} (x, x')), M_{\lambda}(y,y')).
\end{equation}
We  consider  loss functions of the form $l(p(x),y) := h(p(x)) - y p(x)$, which includes standard choices such as the logistic loss and the cross-entropy loss. Denote $L_n^{std} := \frac{1}{n} \sum_{i=1}^n l(p(x_i),y_i)$ and let  $\mathcal{D}_x$ be the empirical distribution of  training samples $\{x_i\}_{i \in [n]}$.  We shall show that NFM exhibits a natural form of implicit regularization, i.e., regularization imposed
implicitly by the stochastic learning strategy or approximation algorithm, without explicitly modifying the loss.

Recall from the main paper that the NFM loss function to be minimized is
\begin{equation}
    L_n^{NFM} = \frac{1}{n^2} \sum_{i=1}^n \sum_{j=1}^n
      \mathbb{E}_{\lambda 
    \sim Beta(\alpha, \beta)}  \mathbb{E}_{\vecc{\xi} \sim \mathcal{Q}} l(p(M_{\lambda, \vecc{
\xi}} (x_i, x_j)), M_{\lambda}(y_i,y_j)), \label{emp_NFMloss}
\end{equation}
where $l$ is a loss function, $\vecc{\xi} := (\xi^{add}, \xi^{mult})$ are drawn from some probability distribution $\mathcal{Q}$ with finite first two moments (with zero mean), and
\begin{align}
    M_{\lambda, \vecc{
\xi}} (x, x') &:= (\mathds{1}+\sigma_{mult} \xi^{mult}) \odot M_\lambda(x, x') + \sigma_{add} \xi^{add}.
\end{align} 
We emphasize that this loss is computed using only the clean data. 

Let $\epsilon > 0$ be a small parameter. In the sequel, we assume that  $l$ is a loss function of the form $l(p(x),y) = h(p(x)) - y p(x)$, and rescale $1-\lambda \mapsto \epsilon (1-\lambda)$, $\sigma_{add} \mapsto \epsilon \sigma_{add}$, $\sigma_{mult} \mapsto \epsilon \sigma_{mult}$.
We recall Lemma 2 in \cite{lim2021noisy} (adapted to our setting), which will be useful later, in the following.

\begin{lem} \label{sm_aux}
The NFM  loss  (\ref
{emp_NFMloss}) can be equivalently written as
\begin{equation} \label{eq_12}
    L^{NFM}_n = \frac{1}{n} \sum_{i=1}^n  \mathbb{E}_{\lambda 
\sim \tilde{\mathcal{D}}_{\lambda}  }  \mathbb{E}_{x_r 
\sim \mathcal{D}_x } \mathbb{E}_{\vecc{\xi} \sim \mathcal{Q}} [h(p(x_i + \epsilon e_i^{NFM})) - y_i p(x_i + \epsilon e_i^{NFM})],
\end{equation}
with
\begin{equation} \label{sm_eNFMk}
    e^{NFM}_i =  (\mathds{1}+\epsilon \sigma_{mult} \xi^{mult}) 
    \odot e_i^{mixup} +  e_i^{noise}.
\end{equation}
Here $e_i^{mixup} = (1-\lambda)(x_r - x_i)$ and $e_i^{noise} =  \sigma_{mult} \xi^{mult} \odot x_i + \sigma_{add} \xi^{add}$, with $x_i, x_r \in \mathcal{X}$ and $\lambda \sim Beta(\alpha, \beta)$.
\end{lem}
Note that the explicit mixing of the labels is absent in the formulation of Lemma \ref{sm_aux}. This gives a more natural interpretation of  \eqref{emp_NFMloss}, in the sense that the inputs are sampled from the vicinal distribution induced by certain (data-dependent) random perturbations  and the labels are kept the same.

Denote  $\nabla p$ and $\nabla^2 p$ as the first and second  directional derivative of $p$ with respect to $x$ respectively. By working in the small parameter regime, we can relate the NFM empirical loss $L_n^{NFM}$ to the original loss $L_n^{std}$ and identify the regularizing effects of NFM.  More precisely, we perform a second-order Taylor expansion of the terms in bracket in Eq. \eqref{eq_12} and arrive at the following result.

\begin{thm} \label{thm_implicitreg}
Let $\epsilon > 0$ be a small parameter, and assume that $h$ and $p$ are twice differentiable.  Then, 
\begin{equation} \label{app_taylorexpansion}
    L^{NFM}_n = L_n^{std} + \epsilon R_1 + \epsilon^2 \tilde{R}_2 +  \epsilon^2 \tilde{R}_3 + \epsilon^2  \phi(\epsilon),
\end{equation}
with
\begin{align}
    \tilde{R}_2 &= R_2 + \sigma_{add}^2 R_2^{add} + \sigma_{mult}^2 R_2^{mult}, \\
    \tilde{R}_3 &= R_3 + \sigma_{add}^2 R_3^{add} + \sigma_{mult}^2 R_3^{mult},
\end{align}
where 
\begin{align}
    R_1 &= \frac{\mathbb{E}_{\lambda  \sim \tilde{\mathcal{D}}_\lambda}[1-\lambda]}{n} \sum_{i=1}^n h'(p(x_i)-y_i)  \nabla p(x_i)^T \mathbb{E}_{x_r \sim \mathcal{D}_x}[x_r - x_i],  \\
    R_2 &= \frac{\mathbb{E}_{\lambda  \sim \tilde{\mathcal{D}}_\lambda}[(1-\lambda)^2]}{2n} \sum_{i=1}^n h''(p(x_i)) \nabla p(x_i)^T \nonumber \\
    &\hspace{0.5cm}  \times \mathbb{E}_{x_r \sim \mathcal{D}_x}[(x_r - x_i)(x_r - x_i)^T]  \nabla p(x_i), \\
    R_3 &= \frac{\mathbb{E}_{\lambda  \sim \tilde{\mathcal{D}}_\lambda}[(1-\lambda)^2]}{2n} \sum_{i=1}^n (h'(p(x_i))-y_i) 
    \nonumber \\
    &\hspace{0.5cm}  \times \mathbb{E}_{x_r \sim \mathcal{D}_x}[(x_r - x_i)^T \nabla^2 p(x_i)  (x_r - x_i)], \\
     R_2^{add} &= \frac{1}{2n} \sum_{i=1}^n h''(p(x_i)) \nabla p(x_i)^T \mathbb{E}_{\vecc{\xi}}[\xi^{add}(\xi^{add})^T]  \nabla p(x_i), \\
    R_2^{mult} &= \frac{1}{2n} \sum_{i=1}^n h''(p(x_i)) \nabla p(x_i)^T (\mathbb{E}_{\vecc{\xi}}[\xi^{mult} (\xi^{mult})^T] \odot  x_i x_i^T ) \nabla p(x_i), \\
    R_3^{add} &= \frac{1}{2n} \sum_{i=1}^n (h'(p(x_i))-y_i) \mathbb{E}_{\vecc{\xi}}[(\xi^{add})^T \nabla^2 p(x_i) \xi^{add} ], \\
    R_3^{mult} &= \frac{1}{2n} \sum_{i=1}^n (h'(p(x_i))-y_i)\mathbb{E}_{\vecc{\xi}}[(\xi^{mult} 
    \odot x_i)^T \nabla^2 p(x_i) (\xi^{mult} 
\odot x_i) ],
\end{align}
and $ \phi(\epsilon) = \mathbb{E}_{\lambda \sim \tilde{\mathcal{D}}_{\lambda}  }  \mathbb{E}_{x_r \sim \mathcal{D}_x } \mathbb{E}_{\vecc{\xi} \sim \mathcal{Q}}[\varphi(\epsilon)]$, with $\varphi$ some function such that $\lim_{\epsilon \to 0} \varphi(\epsilon) = 0$.
\end{thm}

Theorem \ref{thm_implicitreg}  implies that, when compared to Manifold Mixup, NFM introduces  additional smoothness, regularizing the directional derivatives, $\nabla p(x_i)$ and $\nabla^2 p(x_i)$, with respect to $x_i$,  according to the noise levels $\sigma_{add}$ and $\sigma_{mult}$, and thus amplifying the regularizing effects of Manifold Mixup and noise injection. In particular, making $\nabla^2 p(x_i)$ small can lead to smooth decision boundaries (at the input level), while reducing the confidence of model predictions. On the other hand, making the 
$\nabla p(x_i)$ small can lead to improvement in model robustness.

\subsection{Implicit Regularization for the Stability Objective} \label{app_subsec:stability}

Recall that $L^{JSD}$ is given by:
\begin{align}            \mathbb{E}_{x, x' \sim \mathcal{D}_x} \mathbb{E}_{x_{am} \sim \mathcal{A}(x), x_{am}' \sim \mathcal{A}(x')}
     \mathbb{E}_{\lambda 
    \sim Beta(\alpha, \beta)}  \mathbb{E}_{\vecc{\xi} \sim \mathcal{Q}} JS_\pi(p(M_{\lambda, \vecc{
\xi}} (x, x')), p(M_{\lambda, \vecc{
\xi}}(x_{am}, x'_{am}) )).
\end{align}

The stability objective to be minimized is
\begin{equation}
    L_n^{JSD} = \frac{1}{n^2} \sum_{i=1}^n \sum_{j=1}^n
      \mathbb{E}_{\lambda 
    \sim Beta(\alpha, \beta)}  \mathbb{E}_{\vecc{\xi} \sim \mathcal{Q}} JS_\pi(p(M_{\lambda, \vecc{
\xi}} (x_i, x_j)), p(M_{\lambda, \vecc{
\xi}} (A(x_i), A(x_j))) ), \label{emp_JSDloss}
\end{equation}
where $\vecc{\xi} := (\xi^{add}, \xi^{mult})$ are drawn from some probability distribution $\mathcal{Q}$ with finite first two moments (with zero mean), and
\begin{align}
    M_{\lambda, \vecc{
\xi}} (x, x') &:= (\mathds{1}+\sigma_{mult} \xi^{mult}) \odot M_\lambda(x, x') + \sigma_{add} \xi^{add}.
\end{align} 
Note that this loss is computed using both the clean data and the  diversely transformed data.

Denote 
\begin{equation}
    \tilde{L}_n^{std} := \frac{1}{n} \sum_{i=1}^n JS_\pi(p(x_i),p(A(x_{i}))),
\end{equation}
and let  $\mathcal{D}_x$ be the empirical distribution of  training samples $\{x_i\}_{i \in [n]}$. 
We shall show that the stability objective exhibits a natural form of implicit regularization.

Let $\epsilon > 0$ be a small parameter, and rescale $1-\lambda \mapsto \epsilon (1-\lambda)$, $\sigma_{add} \mapsto \epsilon \sigma_{add}$, $\sigma_{mult} \mapsto \epsilon \sigma_{mult}$ as before. 
For this stability objective, we have the following  result, derived analogously to that in Lemma \ref{sm_aux}.

\begin{lem} \label{sm_aux2}
The JSD  loss  (\ref
{emp_JSDloss}) can be equivalently written as
\begin{equation} \label{eq_1}
    L^{JSD}_n = \frac{1}{n} \sum_{i=1}^n  \mathbb{E}_{\lambda 
\sim \tilde{\mathcal{D}}_{\lambda}  }  \mathbb{E}_{x_r 
\sim \mathcal{D}_x } \mathbb{E}_{\vecc{\xi} \sim \mathcal{Q}} [JS_\pi(p(x_i + \epsilon e_i^{NFM}), p(A(x_i) + \epsilon e_i^{JSD}))],
\end{equation}
with
\begin{equation} \label{sm_eNFMk}
    e^{NFM}_i =  (\mathds{1}+\epsilon \sigma_{mult} \xi^{mult}) 
    \odot e_i^{mixup} +  e_i^{noise},
\end{equation}
and 
\begin{equation} \label{sm_eNFMk}
    e^{JSD}_i =  (\mathds{1}+\epsilon \sigma_{mult} \xi^{mult}) 
    \odot \tilde{e}_i^{mixup} +  \tilde{e}_i^{noise}.
\end{equation}
Here $e_i^{mixup} = (1-\lambda)(x_r - x_i)$, $e_i^{noise} =  \sigma_{mult} \xi^{mult} \odot x_i + \sigma_{add} \xi^{add}$,  $\tilde{e}_i^{mixup} = (1-\lambda)(A(x_r) - A(x_i))$ and $\tilde{e}_i^{noise} =  \sigma_{mult} \xi^{mult} \odot A(x_i) + \sigma_{add} \xi^{add}$, with $x_i, x_r \in \mathcal{X}$, $A \sim \mathcal{A}$ and $\lambda \sim Beta(\alpha, \beta)$. \\
\end{lem}

The following proposition will be useful later. \\

\begin{prop} \label{prop_sm}
Let $\epsilon > 0$ be a small parameter.
For $i = 0,1,\dots, M$, denote $\tilde{p}_i := p(z_{i}+R_i(\epsilon))$, $p_i := p(z_{i})$, $\nabla p_i := \nabla p(z_{i})$ and $\nabla^2 p_i := \nabla^2 p(z_{i})$, with the $z_{i}, \delta z_{i} \in \mathcal{X}$, and $R_i$ some twice differentiable function such that  $R_i(0) = 0$, $R_i'(0) = \delta z_i$, and $R_i''(0) = \delta^2 z_i$. Suppose that $p: \mathcal{X} \to \Delta^{K-1}$ is twice differentiable. Then, we have, 
\begin{align}
    &JS_\pi(\tilde{p}_0, \dots, \tilde{p}_M) \nonumber  \\
    &= JS_\pi(p_0, \dots, p_M) + \epsilon \sum_{k=1}^K \sum_{l=0}^M \pi_l \left( (\nabla p^{k}_l)^T \delta z_l  \cdot \ln\left( \frac{p_l^k}{\sum_{j=0}^M \pi_j p_j^k} \right) \right)  \nonumber \\ 
    &\ \ \ \ + \frac{\epsilon^2}{2} \sum_{k=1}^K \sum_{l=0}^M \pi_l \left( ( (\delta z_{l})^T (\nabla^2 p_l^k) \delta z_{l} + (\nabla p^k_l)^T \delta^2 z_l) \cdot \ln\left( \frac{p_l^k}{\sum_{j=0}^M \pi_j p_j^k} \right) + \frac{( (\nabla p_l^k)^T \delta z_{l} )^2 }{p_l^k} \right) \nonumber \\
    &\ \ \ \ - \frac{\epsilon^2}{2} \sum_{k=1}^K \frac{(\sum_{l=0}^M \pi_l (\nabla p_l^k)^T \delta z_{l})^2 }{\sum_{l=0}^M \pi_l p_l^k} + o(\epsilon^3), \label{taylored}
\end{align}
as $\epsilon \to 0$.
\end{prop}

\begin{proof} 

By definition, 
\begin{align}
JS_\pi(\tilde{p}_0, \dots, \tilde{p}_M) &:= H\left(\sum_{l=0}^M \pi_l \tilde{p}_l \right) - \sum_{l=0}^M \pi_l H(\tilde{p}_l) \\
&= -\sum_{k=1}^K \left(\sum_{l=0}^M \pi_l \tilde{p}_l^k \right) \ln\left( \sum_{l=0}^M \pi_l \tilde{p}_l^k \right) + \sum_{l=0}^M \pi_l \sum_{k=1}^K \tilde{p}_l^k \ln(\tilde{p}_l^k) \\
&= -\sum_{k=1}^K \left(\sum_{l=0}^M \pi_l p^k(z_{l}+ R_l(\epsilon)) \right) \ln\left( \sum_{l=0}^M \pi_l p^k(z_{l}+ R_l(\epsilon)) \right) \nonumber \\
&\ \ \ \ \ + \sum_{l=0}^M \pi_l \sum_{k=1}^K p^k(z_{l}+ R_l(\epsilon)) \ln(p^k(z_{l}+ R_l(\epsilon))) \\
&= \sum_{l=0}^M  \sum_{k=1}^K \pi_l A_l^k(\epsilon),
\end{align}
where 
\begin{equation} \label{eq_jsdetail}
    A_l^k(\epsilon) := p^k(z_l + R_l(\epsilon)) \ln\left(\frac{p^k(z_l + R_l(\epsilon))}{\sum_j \pi_j p^k(z_j + R_j(\epsilon))} \right).
\end{equation}

Then, using twice-differentiability of $p$, we perform a second-order Taylor expansion on \eqref{eq_jsdetail} in the small parameter $\epsilon$: \begin{equation} \label{taylor}
    A_l^k(\epsilon) =  A_l^k(0) + \epsilon  (A_l^k)'(0) + \frac{1}{2}\epsilon^2 (A_l^k)''(0) + o(\epsilon^2),  
\end{equation}
as $\epsilon \to 0$.

It remains to compute $(A_l^k)'(0)$ and $(A_l^k)''(0)$. 

Using chain rule, we obtain:
\begin{align}
    (A_l^k)'(\epsilon) &= (\nabla p^k(z_l + R_l(\epsilon)))^T R_l'(\epsilon) \bigg(1 + \ln\left( \frac{p^k(z_l + R_l(\epsilon))}{\sum_j \pi_j p^k(z_j + R_j(\epsilon))} \right) \bigg) \nonumber \\
    &\ \ \ \ \ - p^k(z_l + R_l(\epsilon)) \frac{\sum_j \pi_j (\nabla p^k(z_j + R_j(\epsilon)) )^T R_j'(\epsilon) }{\sum_l \pi_l p^k(z_l + R_l(\epsilon))}.
\end{align}
Therefore, using the fact that $R_i(0) = 0$ and $R_i'(0) = \delta z_i$, we have:
\begin{align}
    (A_l^k)'(0) &= (\nabla p^k(z_l))^T \delta z_l \bigg(1 + \ln\left( \frac{p^k(z_l)}{\sum_j \pi_j p^k(z_j)} \right) \bigg) - p^k(z_l) \frac{\sum_j \pi_j (\nabla p^k(z_j) )^T \delta z_j }{\sum_l \pi_l p^k(z_l)}.
\end{align}

Similarly, applying chain rule, we obtain:
\begin{align}
    &(A_l^k)''(\epsilon) \nonumber  \\ 
    &= [ R_l'(\epsilon)^T \nabla^2 p^k(z_l + R_l(\epsilon)) R_l'(\epsilon) \nonumber \\
    &\ \ \ \ \ + (\nabla p^k(z_l + R_l(\epsilon)))^T R_l''(\epsilon) ]\bigg(1 + \ln\left( \frac{p^k(z_l + R_l(\epsilon))}{\sum_j \pi_j p^k(z_j + R_j(\epsilon))} \right) \bigg) \nonumber \\
    &\ \ \ \ \ + \frac{(\nabla p^k(z_l + R_l(\epsilon))^T R_l'(\epsilon))^2}{p^k(z_l + R_l(\epsilon))} - \frac{2\nabla p^k(z_l + R_l(\epsilon))^T R_l'(\epsilon) \left(\sum_j \pi_j \nabla p^k(z_j + R_j(\epsilon))^T R_j'(\epsilon) \right) }{\sum_j \pi_j p^k(z_j + R_j(\epsilon))} \nonumber \\
    &\ \ \ \ \ - p^k(z_l + R_l(\epsilon)) \frac{\sum_j \pi_j (R_j'(\epsilon)^T \nabla^2 p^k(z_j + R_j(\epsilon)) R_j'(\epsilon) + \nabla p^k(z_j + R_j(\epsilon))^T R_j''(\epsilon) )   }{\sum_l \pi_l p^k(z_l + R_l(\epsilon))} \nonumber \\
    &\ \ \ \ \ + p^k(z_l + R_l(\epsilon)) \left( \frac{\sum_l \pi_l \nabla p^k(z_l + R_l(\epsilon))^T R_l'(\epsilon) }{\sum_j \pi_j p^k(z_j + R_j(\epsilon))} \right)^2.
\end{align}
Therefore, using the fact that $R_i(0) = 0$, $R_i'(0) = \delta z_i$, and $R_i''(0) = \delta^2 z_i$, we have:
\begin{align}
    (A_l^k)''(0) &= [ (\delta z_l)^T \nabla^2 p^k(z_l) \delta z_l + (\nabla p^k(z_l))^T \delta^2 z_l ]\bigg(1 + \ln\left( \frac{p^k(z_l )}{\sum_j \pi_j p^k(z_j )} \right) \bigg) \nonumber \\
    &\ \ \ \ \ + \frac{(\nabla p^k(z_l)^T \delta z_l)^2}{p^k(z_l)} - \frac{2\nabla p^k(z_l)^T \delta z_l \left(\sum_j \pi_j \nabla p^k(z_j)^T \delta z_j \right) }{\sum_j \pi_j p^k(z_j)} \nonumber \\
    &\ \ \ \ \ - p^k(z_l) \frac{\sum_j \pi_j ((\delta z_j)^T \nabla^2 p^k(z_j) \delta z_j + \nabla p^k(z_j)^T \delta^2 z_j )   }{\sum_l \pi_l p^k(z_l)} \nonumber \\
    &\ \ \ \ \ + p^k(z_l) \left( \frac{\sum_l \pi_l \nabla p^k(z_l)^T \delta z_j }{\sum_j \pi_j p^k(z_j)} \right)^2.
\end{align}

Combining the above results with the Taylor expansion \eqref{taylor} and simplifying the resulting expression lead to \eqref{taylored}. 

\end{proof}

Specializing Proposition \ref{prop_sm} to our setting, we arrive at the following result. \\

\begin{thm} \label{thm_jsd}
Let $\epsilon > 0$ be a small parameter, and assume that $p$ is twice differentiable. Define, for $k \in [K]$,
\begin{equation}
    q^k(x, y) := \nabla^2 p^k(x) \ln\left( \frac{2p^k(x)}{p^k(x)+p^k(y)} \right) + \nabla p^k(x) (\nabla p^k(x))^T \frac{p^k(y)}{p^k(x)(p^k(x)+p^k(y))}.  
\end{equation}
Then, 
\begin{equation}
    L_n^{JSD} = \tilde{L}_n^{std} + \epsilon S_1 + \epsilon^2 \tilde{S}_2 + \epsilon^2 \phi(\epsilon),
\end{equation}
with 
\begin{equation} \label{S2}
    \tilde{S}_2 = S_2 + \sigma_{add}^2 S_2^{add} + \sigma_{mult}^2 S_2^{mult}, 
\end{equation}
where
\begin{align}
    S_1 &= \frac{\mathbb{E}_\lambda[1-\lambda]}{2n} \sum_{i=1}^n \sum_{k=1}^K \bigg( (\nabla p^k(x_i))^T (\mathbb{E}_{x_r}[x_r] - x_i) \cdot \ln\left(\frac{2p^k(x_i)}{p^k(x_i)+p^k(A(x_i))} \right) \nonumber \\
    &\ \ \ \ + (\nabla p^k(A(x_i)))^T (\mathbb{E}_{x_r}[A(x_r)] - A(x_i)) \cdot \ln\left(\frac{2p^k(A(x_i))}{p^k(x_i)+p^k(A(x_i))} \right) \bigg) \\
    S_2 &= \frac{\mathbb{E}_\lambda[(1-\lambda)^2] }{4n}  \sum_{i=1}^n \sum_{k=1}^K \bigg( \mathbb{E}_{x_r}[(x_r-x_i)^T q^k(x_i, A(x_i))(x_r-x_i)] \nonumber \\ 
    &\ \ \ \ + \mathbb{E}_{x_r}[(A(x_r)-A(x_i))^T q^k(A(x_i),x_i)(A(x_r)-A(x_i))]  \nonumber \\
     &\ \ \ \  -  \frac{ 2(\nabla p^k(x_i))^T \mathbb{E}_{x_r}[(x_r-x_i)(A(x_r)-A(x_i))^T] \nabla p^k(A(x_i)) }{p^k(x_i) + p^k(A(x_i)) } \bigg), \\
    S_2^{add} &= \frac{1}{4n}  \sum_{i=1}^n \sum_{k=1}^K \bigg( \mathbb{E}_{\vecc{\xi}}[(\xi^{add})^T q^k(x_i, A(x_i)) \xi^{add} ]  + \mathbb{E}_{\vecc{\xi}}[(\xi^{add})^T q^k( A(x_i), x_i) \xi^{add} ] \nonumber \\
    &\ \ \ \ - \frac{2 (\nabla p^k(x_i))^T \mathbb{E}_{\xi}[\xi^{add} (\xi^{add})^T ] \nabla p^k(A(x_i)) }{p^k(x_i) + p^k(A(x_i)) } \bigg), \\
    S_2^{mult} &= \frac{1}{4n}  \sum_{i=1}^n \sum_{k=1}^K \bigg( \mathbb{E}_{\vecc{\xi}}[(\xi^{mult} \odot x_i)^T q^k(x_i, A(x_i)) (\xi^{mult} \odot x_i) ] \nonumber \\
    &\ \ \ \  + \mathbb{E}_{\vecc{\xi}}[(\xi^{mult} \odot A(x_i))^T q^k(A(x_i), x_i) (\xi^{mult} \odot A(x_i)) ] \nonumber \\
    &\ \ \ \ - \frac{2 (\nabla p^k(x_i))^T \mathbb{E}_{\xi}[(\xi^{mult} \odot x_i) (\xi^{mult} \odot A(x_i) )^T ] \nabla p^k(A(x_i)) }{p^k(x_i) + p^k(A(x_i)) } \bigg),
\end{align}
and $\phi(\epsilon) =\mathbb{E}_{\lambda \sim \tilde{\mathcal{D}}_{\lambda}  }  \mathbb{E}_{x_r \sim \mathcal{D}_x } \mathbb{E}_{\vecc{\xi} \sim \mathcal{Q}}[\varphi(\epsilon)]$, with $\varphi$ some function such that $\lim_{\epsilon \to 0} \varphi(\epsilon) = 0$. 
\end{thm}
\begin{proof}
By assumption, $p$ is twice differentiable, so  we apply Proposition \ref{prop_sm} (with $M=1$,  $K=2$, $\pi_0 = \pi_1 = 1/2$, $z_0 = x_i$, $z_1 = A(x_i)$,  $R_0(\epsilon) = \epsilon (e_i^{mixup} + e_i^{noise}) + \epsilon^2 \sigma_{mult} \xi^{mult} \odot e_i^{mixup}$, and $R_1(\epsilon) = \epsilon (\tilde{e}_i^{mixup} + \tilde{e}_i^{noise}) + \epsilon^2 \sigma_{mult} \xi^{mult} \odot \tilde{e}_i^{mixup}$) to Eq. \eqref{eq_12}. Simplifying (using linearity of expectation, and the expression for the $e_i^{mixup}$, $\tilde{e}_i^{mixup}$, $e_i^{noise}$ and $\tilde{e}_i^{noise}$given in Lemma \ref{sm_aux2}) and rearranging the resulting expression then lead to the expression given in the theorem.
\end{proof}

Theorem \ref{thm_jsd}  implies that, when compared to the stability training in AugMix, the stability training in \textit{NoisyMix}   introduces  additional smoothness, regularizing the directional derivatives, $\nabla p(x_i)$, $\nabla p(A(x_i))$, $\nabla^2 p(x_i)$ and $\nabla^2 p(A(x_i))$, with respect to both the clean data $x_i$ and augmented data $A(x_i)$,  according to the mixing coefficient and the noise levels. 

Combining this with the analogous interpretation of Theorem \ref{thm_implicitreg} and recalling that $L_n^{\textit{NoisyMix} } = L_n^{NFM} + \gamma L_n^{stability}$, we see that \textit{NoisyMix}  amplifies the regularizing effects of AugMix, and can lead to smooth decision boundaries and improvement in model robustness. Since \textit{NoisyMix}  implicitly reduces the $\nabla p(x_i)$ and $\nabla p(A(x_i))$, following the argument in Section C of  \cite{lim2021noisy}, we see that \textit{NoisyMix}  implicitly increases the classification margin, thereby making the model more robust to input perturbations. We shall study another lens in which \textit{NoisyMix}  can improve robustness in the next section.

\section{Robustness of NoisyMix}
\label{sm_secD}

We now demonstrate how \textit{NoisyMix}  training can lead to a model that is both adversarially robust and  stable.


We consider the binary cross-entropy loss, setting $h(z) = \log(1+e^z)$, with the labels $y$ taking value in $\{0,1\}$ and the  classifier model $p: \RR^d \to \RR$. In the following, we assume that the model parameter  $\theta \in \Theta := \{\theta : y_i p(x_i) + (y_i - 1) p(x_i) \geq 0 \text{ for all } i \in [n] \}$. We remark that this set contains the set of all parameters with correct classifications of training samples (before applying \textit{NoisyMix}), since $ \{\theta : 1_{\{p(x_i)  \geq 0\}} = y_i \text{ for all } i \in [n] \} \subset \Theta$. Therefore, the condition of $\theta \in \Theta$ is fulfilled when the model classifies all labels correctly for the training data before applying \textit{NoisyMix}. Since the training error often becomes zero in finite time in practice, we shall study the effect of \textit{NoisyMix}  on model robustness in the regime of $\theta \in \Theta$. 

Working in the data-dependent parameter space $\Theta$, we  obtain the following result. \\

\begin{thm} \label{thm_advbound}
Let $\theta \in \Theta := \{\theta : y_i p(x_i) + (y_i - 1) p(x_i) \geq 0 \text{ for all } i \in [n] \}$ be a point such that $\nabla p(x_i)$, $\nabla p(A(x_i))$, $\nabla^2 p(x_i)$ and $\nabla^2 p(A(x_i))$ exist for all $i \in [n]$. Assume that $p(x_i) = \nabla p(x_i)^T x_i$, $\nabla^2 p(x_i) = 0$ for all $i \in [n]$. In addition, suppose that $\mathbb{E}_{r \sim \mathcal{D}_x}[r] = \mathbb{E}_{r \sim \mathcal{D}_x}[A(r)] = 0$ and $\|x_i\|_2 \geq c_x \sqrt{d}$ for all $i \in [n]$. Then, 
\begin{align}
    L_n^{\textit{NoisyMix} } \geq \frac{1}{n} \sum_{i=1}^n \max_{\|\delta_i \|_2 \leq \epsilon_i^{mix}} l(p(x_i + \delta_i), y_i) + \gamma \tilde{L}_{n}^{JSD} +  \epsilon^2 L_n^{reg} +   \epsilon^2  \phi(\epsilon),
\end{align}
where 
\begin{align}
    \epsilon_i^{mix} &= \epsilon \mathbb{E}_{\lambda \sim \tilde{\mathcal{D}}_\lambda}[1-\lambda] \cdot c_x \sqrt{d} \ |\cos(\nabla p(x_i), x_i)|, \\
    \tilde{L}_n^{JSD} &= \frac{1-\epsilon \mathbb{E}_{\lambda \sim \tilde{\mathcal{D}}_\lambda }[1-\lambda] }{n} \sum_{i=1}^n JS_\pi(p(x_i),p(A(x_{i}))), \\
    L_n^{reg} &= \frac{1}{2n} \sum_{i=1}^n |h''(p(x_i))|  (\epsilon_i^{reg})^2 + \gamma \tilde{S}_2,
\end{align}
with $\tilde{S}_2$ given  in \eqref{S2} and
\begin{align}
    (\epsilon_i^{reg})^2 &=  \|\nabla p(x_i)\|_2^2 \bigg( \mathbb{E}_\lambda[(1-\lambda)]^2 \mathbb{E}_{x_r}[\|x_r\|_2^2 \cos(\nabla p(x_i), x_r)^2 ] \nonumber \\
    &\hspace{0.7cm} + \sigma_{add}^2 \mathbb{E}_{\vecc{\xi}}[\|\xi_{add}\|_2^2 \cos(\nabla p(x_i),\xi_{add})^2]  \nonumber \\ 
    &\hspace{0.7cm} + \sigma_{mult}^2 \mathbb{E}_{\vecc{\xi}}[ \|\xi_{mult} \odot x_i \|_2^2 \cos(\nabla p(x_i),\xi_{mult} \odot x_i)^2  ]  \bigg),
\end{align}
and $\phi$ is some function such that $\lim_{\epsilon \to 0} \phi(\epsilon) = 0$.
\end{thm}

\begin{proof}[Proof of Theorem \ref{thm_advbound}] The inequality in the theorem follows upon applying Theorem 2 in \cite{lim2021noisy} together with Theorem \ref{thm_jsd}.
\end{proof}

We remark that the assumption made in Theorem \ref{thm_advbound} is similar to the one made in \cite{lamb2019interpolated,zhang2020does}, and is satisfied by linear models and  deep neural networks with ReLU activation function and max-pooling (for a proof of this, we refer to Section B.2 in \cite{zhang2020does}).

Theorem \ref{thm_advbound}
says that $L_n^{\textit{NoisyMix} }$ is approximately an upper bound of sum of an adversarial loss with $l_2$-attack of size $\epsilon^{mix} = \min_i \epsilon_i^{mix}$, a stability objective, and a data-dependent regularizer. Therefore, {\it minimizing the \textit{NoisyMix}  loss would result in a small regularized adversarial loss, while promoting stability for the model}. This not only amplifies the robustness benefits of Manifold Mixup, but also imposes additional smoothness, due to the noise injection in NFM and  stability training on the noise-perturbed transformed data.
The latter can also help reduce robust overfitting and improve test performance  \cite{rice2020overfitting, rebuffi2021fixing}.

\newpage
\section{Additional Experiments and Details} \label{sm_secG}

In this section we show additional results to support and complement the findings in Section~\ref{sec:results}.

\subsection{Additional ImageNet-C, CIFAR-10-C and CIFAR-100-C Results}

ImageNet-C, CIFAR-10-C, and CIFAR-100-C consist of 15 different corruption types that are used to evaluate the robustness of a given model. These corruption types are illustrated in Figure~\ref{fig:cimages}. Here we show examples that corresponds to severity level 5, i.e., the most severe perturbation level. The considered corruptions cover a wide verity of perturbations that can potentially occur in real-world situations (i.e., a self-driving care might need to navigate in a snow storm).

\begin{figure}[!h]
	\centering
	\includegraphics[width=0.63\textwidth]{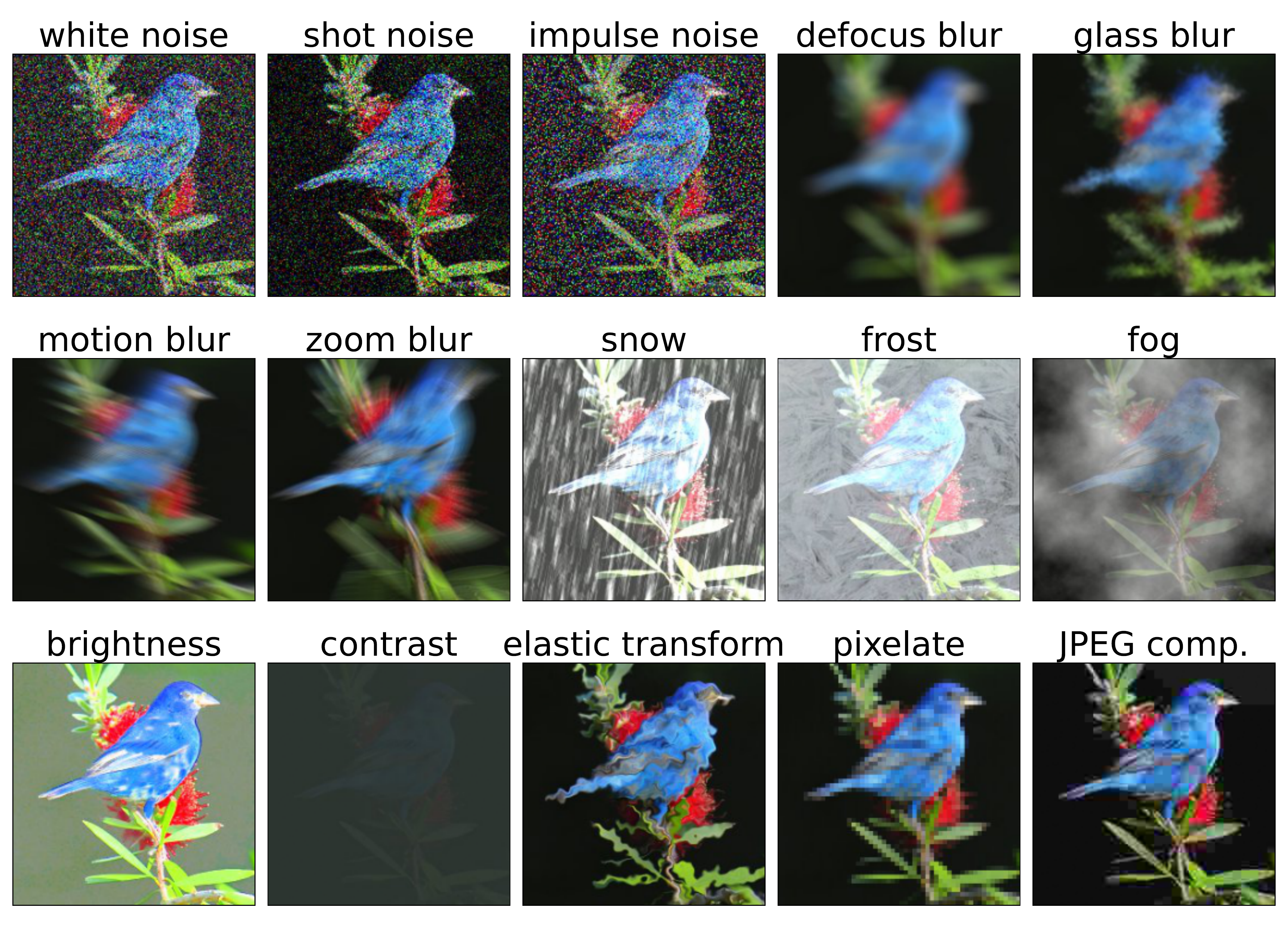}
	\vspace{-0.4cm}
	\caption{The ImageNet-C example shows the 15 different corruption types at severity level 5.}
	\label{fig:cimages}
\end{figure}

\paragraph{ImageNet-C Results.}
Table~\ref{tab:results_imagenet_c} shows the robust accuracy with respect to each perturbation type for the different ResNet-50 models trained on ImageNet-1k.
\textit{NoisyMix} performs best on 9 out of the 15 corruption types. The advantage is particularly pronounced for noise perturbations and weather corruptions. 
To illustrate this further, we compute the average robustness accuracy for the 4 meta categories (i.e., noise, weather, blur, and digital corruptions). The results are shown in Figure~\ref{fig:imagenet_type}, where it can be seen that \textit{NoisyMix} is the most robust model with respect to noise, weather, and digital corruptions, while AugMix is slightly more robust to blur corruptions.
\begin{sidewaystable}
	\caption{Detailed results for ResNet-50s trained with various data augmentation schemes and evaluated on ImageNet-C. All values indicate the robust accuracy (in \%) as a function of the considered 15 perturbation types (higher values are better).}
	\label{tab:results_imagenet_c}
	\centering
	\scalebox{0.85}{
		\setlength\tabcolsep{3pt}
		\begin{tabular}{l|ccc|cccc|cccc|cccccccc}
			\toprule
			& \multicolumn{3}{|c|}{Noise} & \multicolumn{4}{|c|}{Blur} & \multicolumn{4}{|c|}{Weather} & \multicolumn{4}{|c}{Digital} &  \\	
			& White & Shot & Impulse & Defocus & Glass & Motion & Zoom & Snow & Frost & Fog & Bright & Contrast & Elastic & Pixel & JPEG \\
			\hline
			Baseline 					& 29.3&	27.0&	23.8&	38.8&	26.8&	38.7&	36.2&	32.5&	38.1&	45.8&	68.0&	39.1&	45.3&	44.8&	53.4 \\
			Adversarial Trained 		& 22.4&	21.0&	13.1&	24.6&	34.8&	33.7&	35.2&	30.6&	29.6&	6.2&	54.4&	8.5&	49.4&	57.3&	60.5\\
			Stylized ImageNet 			&  41.3&	40.3&	37.2&	42.9&	32.3&	45.5&	35.8&	41.0&	41.6&	47.0&	67.4&	43.3&	49.5&	55.6&	57.7 \\
			Fast AutoAugment 			& 40.9&	40.3&	36.3&	42.0&	35.1&	41.2&	36.3&	40.2&	43.0&	52.9&	\textbf{72.0}&	51.7&	47.5&	48.5&	57.7\\
			Mixup 						& 40.5&	36.9&	34.2&	41.9&	29.1&	43.9&	41.1&	42.0&	45.5&	\textbf{57.5}& {71.4}&	51.0&	47.1&	51.8&	58.6 \\
			Manifold Mixup 				&  36.3&	34.2&	31.0&	39.7&	27.6&	42.0&	40.2&	38.6&	49.0&	{54.6}&	69.1&	51.3&	45.5&	45.7&	54.2 \\
			CutMix 						& 36.0&	34.1&	28.4&	40.2&	25.4&	40.6&	36.9&	34.1&	39.6&	50.3&	70.3&	46.8&	45.2&	36.4&	51.2 \\
			Puzzle Mix 					& 22.4&	21.0&	13.1&	24.6&	34.8&	33.7&	35.2&	30.6&	29.6&	6.2&	54.4&	8.5&	49.4&	\textbf{57.3}&	60.5\\
			AugMix 						&  40.6&	41.1&	37.7&	\textbf{47.7}&	34.9&	\textbf{53.5}&	\textbf{49.0}&	39.9&	43.8&	47.1&	69.5&	51.1&	52.0&	57.0&	60.3\\
			\textit{NoisyMix}  (ours) 	& \textbf{53.1}&	\textbf{52.5}&	\textbf{51.7}&	47.1&	\textbf{37.5} &	52.2&	47.2&	\textbf{45.0}&	\textbf{52.6}&	52.9&	71.1&	\textbf{52.6}&	\textbf{52.6}&	53.2&	\textbf{63.8}	 \\
			
			\bottomrule
	\end{tabular}}
\end{sidewaystable}

\begin{figure}[!t]
	\centering
	\includegraphics[width=0.65\textwidth]{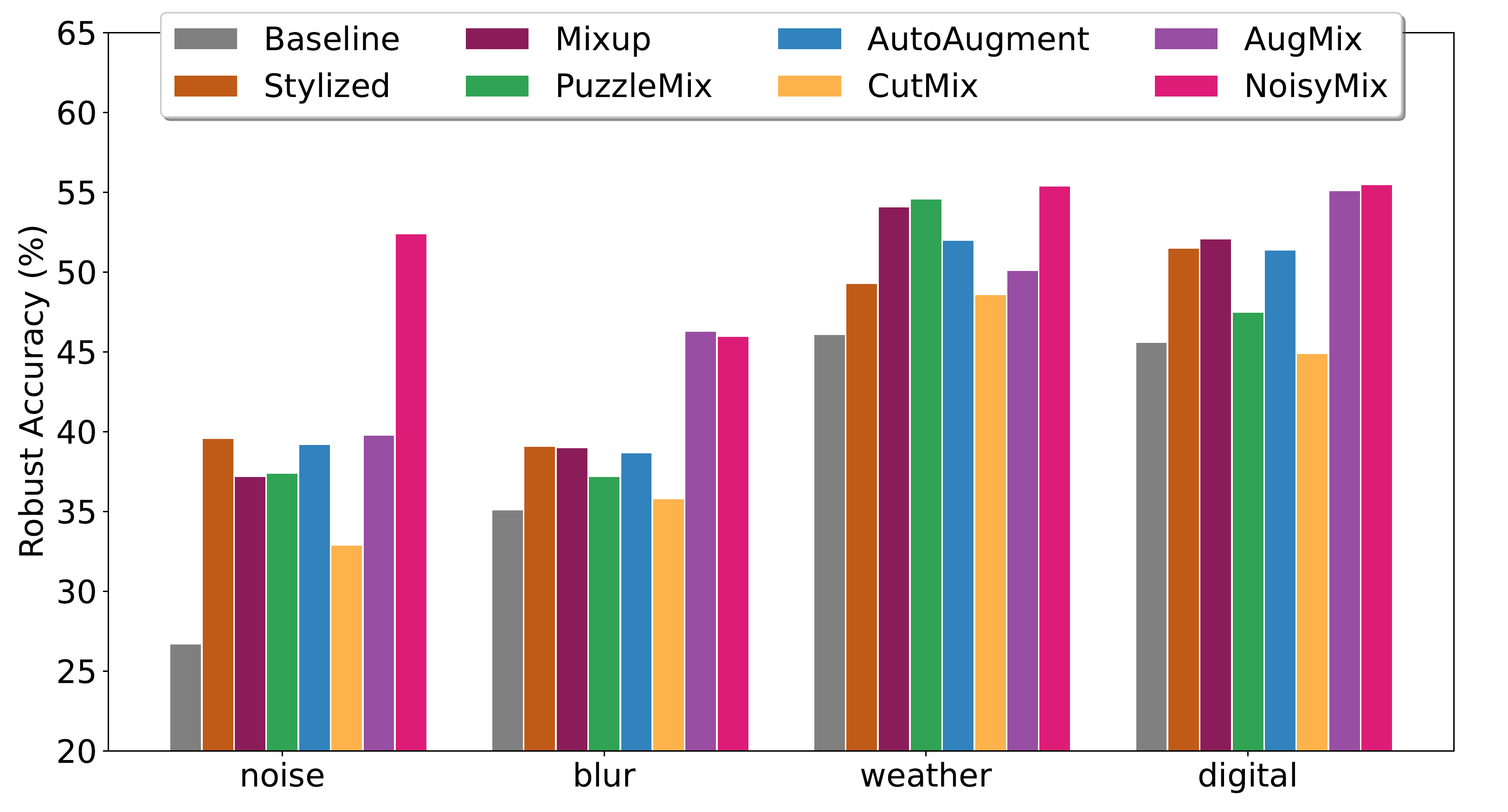}
	\vspace{-0.4cm}
	\caption{Average robust accuracy of ResNet-50 models for the 4 meta categories noise, blur, weather, and digital distortions. \textit{NoisyMix} dominates in 2 out of the 4 categories, while the robustness to digital corruptions is nearly on par with AugMix.} 
	\label{fig:imagenet_type}
\end{figure}

\begin{sidewaystable}
	\caption{Detailed results for ResNet-18s trained with various data augmentation schemes and evaluated on CIFAR-10-C. All values indicate the robust accuracy (in \%) as a function of the considered 15 perturbation types (higher values are better).}
	\label{tab:results_cifar10_c}
	\centering
	\scalebox{0.85}{
		\setlength\tabcolsep{3pt}
		\begin{tabular}{l|ccc|cccc|cccc|cccccccc}
			\toprule
			& \multicolumn{3}{|c|}{Noise} & \multicolumn{4}{|c|}{Blur} & \multicolumn{4}{|c|}{Weather} & \multicolumn{4}{|c}{Digital} &  \\	
			& White & Shot & Impulse & Defocus & Glass & Motion & Zoom & Snow & Frost & Fog & Bright & Contrast & Elastic & Pixel & JPEG \\
			\hline
			Baseline 					& 46.4&	59.2&	52.6&	82.6&	53.1&	78.0&	77.5&	82.6&	78.3&	88.4&	93.8&	76.7&	84.4&	74.5&	79.3 \\
			Mixup 						& 58.6&	68.3&	53.9&	88.2&	69.2&	84.0&	84.0&	88.5&	87.9&	91.2&	94.4&	88.5&	87.6&	79.8&	82.1 \\
			Manifold Mixup 				& 55.5&	66.4&	48.2&	85.7&	60.5&	82.3&	82.5&	87.3&	84.9&	90.7&	94.5&	85.1&	86.2&	75.3&	80.9  \\
			CutMix 						& 29.5&	42.2&	50.2&	83.7&	58.7&	81.0&	78.4&	87.7&	80.9&	90.7&	94.8&	83.5&	86.4&	75.3&	71.5\\
			Puzzle Mix 					& 76.5&	81.7&	69.0&	86.4&	77.5&	83.3&	82.0&	90.2&	90.1&	91.6&	\textbf{95.3}&	87.2&	88.8&	85.3&	77.0 \\
			NFM 						& 78.7&	83.5&	73.4&	85.4&	72.0&	80.3&	83.2&	89.0&	89.5&	87.1&	94.2&	76.2&	86.2&	82.5&	87.9 \\			
			AugMix 						& 79.8&	85.2&	85.5&	\textbf{94.3} &	79.6&	\textbf{92.4} &	\textbf{93.3}&	89.7&	89.4&	\textbf{92.0}&	{94.6}&	\textbf{90.7} &	90.6&	87.7&	87.7 \\
			\textit{NoisyMix}  (ours) 	& \textbf{89.9} &	\textbf{91.8} &	\textbf{94.6} &	93.6&	\textbf{85.9}&	91.6&	92.7&	\textbf{91.1} &	\textbf{91.7}&	91.1&	94.5&	88.1&	\textbf{90.9}&	\textbf{89.3}&	\textbf{90.7}	 \\
			\bottomrule
	\end{tabular}}
\end{sidewaystable}

\begin{sidewaystable}
	\caption{Detailed results for ResNet-18s trained with various data augmentation schemes and evaluated on CIFAR-100-C. All values indicate the robust accuracy (in \%) as a function of the considered 15 perturbation types (higher values are better).}
	\label{tab:results_cifar100_c}
	\centering
	\scalebox{0.85}{
		\setlength\tabcolsep{3pt}
		\begin{tabular}{l|ccc|cccc|cccc|cccccccc}
			\toprule
			& \multicolumn{3}{|c|}{Noise} & \multicolumn{4}{|c|}{Blur} & \multicolumn{4}{|c|}{Weather} & \multicolumn{4}{|c}{Digital} &  \\	
			& White & Shot & Impulse & Defocus & Glass & Motion & Zoom & Snow & Frost & Fog & Bright & Contrast & Elastic & Pixel & JPEG \\
			\hline
			Baseline 					& 22.1&	30.3&	24.6&	60.2&	21.3&	54.5&	53.8&	54.8&	48.9&	64.5&	73.4&	55.1&	60.6&	51.8&	50.6\\
			Mixup 						& 27.2&	36.0&	21.7&	65.9&	25.8&	62.3&	60.7&	63.7&	57.4&	70.2&	75.9&	67.9&	64.8&	58.7&	55.5 \\
			Manifold Mixup 				& 24.7&	33.6&	27.0&	65.1&	23.1&	60.8&	60.3&	63.8&	57.3&	69.2&	\textbf{76.7}&	62.8&	64.6&	57.2&	54.2\\
			CutMix 						& 13.1&	20.6&	29.8&	62.4&	22.3&	56.0&	56.5&	59.8&	49.6&	65.4&	74.2&	56.5&	61.6&	45.4&	45.7 \\
			Puzzle Mix 					& 47.3&	54.3&	40.1&	64.1&	33.1&	59.6&	59.0&	64.5&	59.4&	68.0&	75.7&	60.7&	64.4&	56.6&	54.4\\
			NFM 						& 50.2&	57.5&	42.1&	64.8&	34.9&	59.8&	61.0&	66.0&	64.7&	65.8&	76.0&	57.1&	65.3&	64.7&	64.8 \\			
			AugMix 						& 46.8&	54.9&	61.5&	74.9&	50.9&	72.2&	73.0&	65.9&	63.7&	67.1&	74.2&	67.8&	68.3&	66.3&	61.5\\
			\textit{NoisyMix}  (ours) 	& \textbf{65.3}&	\textbf{69.2}&	\textbf{77.8}&	\textbf{76.8}&	\textbf{58.6}&	\textbf{73.7}&	\textbf{75.4}&	\textbf{71.1}&	\textbf{69.7}&	\textbf{71.0}&	\textbf{76.7}&	\textbf{68.0}&	\textbf{72.2}&	\textbf{70.5}&	\textbf{68.9}	 \\
			
			\bottomrule
	\end{tabular}}
\end{sidewaystable}

\paragraph{CIFAR-C Results.}
Table~\ref{tab:results_cifar10_c} shows the detailed robustness properties of ResNet-18 models trained with different data augmentation schemes and evaluated on CIFAR-10-C. Table~\ref{tab:results_cifar100_c} shows similar results for ResNet-18 models that are evaluated on CIFAR-100-C.
Overall, \textit{NoisyMix} performs best across most of the different input perturbations and data corruptions as compared to various baselines in the data augmentation space. In general, the combination of a noisy training scheme paired with stability training appears to prescribe our model a bias towards stochastic noise due to the nature of fitting a reconstruction over a stochastic representation. 
This can be seen by the good robust accuracy on the noise, weather and digital perturbations being higher for \textit{NoisyMix} and lower in the other methods that place less encompassing emphasis on stochastic perturbations. Overall, the advantage of training models with \textit{NoisyMix} is more pronounced for the CIFAR-100 task, where \textit{NoisyMix} dominates AugMix. 

The limitations of certain neural network architectures in terms of their abilities to learn high-frequency features in low-dimensional domains~\cite{tancik2020fourier} could also imply that our model will tend towards learning features similar in frequency to the in-domain task, and hence perform most accurately in specific out-of-domain perturbations on the same dataset containing similar frequencies. We leave this to future work.

\begin{sidewaystable}
	\caption{Detailed results for Wide-ResNet models trained with data augmentation schemes and evaluated on CIFAR-10-C. All values indicate the robust accuracy (in \%) as a function of the considered 15 perturbation types (higher values are better).}
	\label{tab:results_wrs_cifar10_c}
	\centering
	\scalebox{0.85}{
		\setlength\tabcolsep{3pt}
		\begin{tabular}{l|ccc|cccc|cccc|cccccccc}
			\toprule
			& \multicolumn{3}{|c|}{Noise} & \multicolumn{4}{|c|}{Blur} & \multicolumn{4}{|c|}{Weather} & \multicolumn{4}{|c}{Digital} &  \\	
			& White & Shot & Impulse & Defocus & Glass & Motion & Zoom & Snow & Frost & Fog & Bright & Contrast & Elastic & Pixel & JPEG \\
			\hline
			Baseline 					& 46.5&	59.6&	50.7&	83.0&	57.3&	79.0&	78.4&	85.1&	81.0&	89.2&	94.6&	77.4&	85.4&	76.5&	79.4 \\
			Mixup 						& 56.5&	67.9&	53.4&	88.7&	69.8&	84.9&	85.4&	90.4&	90.2&	91.9&	95.3&	87.9&	88.3&	78.8&	83.3 \\
			Manifold Mixup 				& 48.0&	61.1&	56.4&	85.7&	60.1&	81.6&	82.3&	89.0&	85.9&	91.4&	95.4&	83.3&	86.4&	76.3&	81.7 \\
			CutMix 						& 44.2&	36.3&	46.6&	83.2&	57.0&	80.7&	76.4&	87.7&	81.3&	91.6&	95.1&	85.5&	85.3&	72.6&	70.4\\
			Puzzle Mix 					& 72.6&	79.1&	58.6&	85.9&	72.5&	83.7&	80.8&	91.2&	89.8&	92.8&	\textbf{95.8}&	89.5&	88.2&	82.8&	75.8\\
			NFM 						& 74.8&	81.2&	71.4&	84.7&	69.6&	78.9&	81.7&	89.7&	89.5&	88.2&	95.0&	73.9&	86.8&	81.9&	87.9\\			
			AugMix 						& 81.0&	86.3&	87.9&	\textbf{95.5}&	80.9&	\textbf{93.9}&	\textbf{94.5}&	91.4&	91.0&	93.3&	95.7&	\textbf{92.1}&	91.8&	89.4&	88.4\\
			\textit{NoisyMix}  (ours) 	& \textbf{87.0}&	\textbf{90.2}&	\textbf{95.7}&	95.3&	\textbf{84.2}&	93.4&	94.4&	\textbf{92.3}&	\textbf{92.2}&	\textbf{93.4}&	\textbf{95.8}&	89.8&	\textbf{92.5}&	\textbf{89.5}&	\textbf{90.8} \\
			\bottomrule
	\end{tabular}}
\end{sidewaystable} 

\begin{sidewaystable}
	\caption{Detailed results for Wide-ResNet models trained with data augmentation schemes and evaluated on CIFAR-100-C. All values indicate the robust accuracy (in \%) as a function of the considered 15 perturbation types (higher values are better).}
	\label{tab:results_wrs_cifar100_c}
	\centering
	\scalebox{0.85}{
		\setlength\tabcolsep{3pt}
		\begin{tabular}{l|ccc|cccc|cccc|cccccccc}
			\toprule
			& \multicolumn{3}{|c|}{Noise} & \multicolumn{4}{|c|}{Blur} & \multicolumn{4}{|c|}{Weather} & \multicolumn{4}{|c}{Digital} &  \\	
			& White & Shot & Impulse & Defocus & Glass & Motion & Zoom & Snow & Frost & Fog & Bright & Contrast & Elastic & Pixel & JPEG \\
			\hline
			Baseline 					& 18.4&	27.0&	19.1&	60.6&	18.4&	54.0&	54.5&	55.7&	48.1&	64.4&	74.9&	54.0&	60.4&	51.9&	50.2 \\
			Mixup 						& 25.3&	34.3&	21.0&	66.5&	23.6&	61.8&	61.3&	64.5&	56.1&	71.1&	77.1&	68.2&	65.4&	56.3&	55.8\\
			Manifold Mixup 				& 20.5&	29.7&	24.2&	65.7&	23.9&	60.9&	60.9&	64.8&	56.8&	70.6&	\textbf{78.2}&	63.3&	65.1&	57.2&	55.3 \\
			CutMix 						& 13.1&	20.8&	31.2&	62.2&	22.6&	57.0&	55.7&	61.2&	52.4&	67.8&	75.4&	59.2&	61.2&	43.7&	44.1\\
			Puzzle Mix 					& 38.6&	47.7&	33.0&	64.4&	29.4&	60.5&	58.7&	66.3&	61.2&	71.2&	77.5&	65.0&	64.9&	54.1&	51.9\\
			NFM 						& 49.4&	57.2&	39.4&	62.9&	35.8&	57.3&	58.3&	66.5&	65.5&	66.0&	76.7&	55.4&	64.3&	62.2&	63.8 \\			
			AugMix 						& 46.6&	55.5&	65.2&	\textbf{78.2}&	49.3&	\textbf{74.8}&	\textbf{76.1}&	69.0&	66.1&	71.5&	77.6&	\textbf{69.9}&	71.2&	66.9&	63.2\\
			\textit{NoisyMix}  (ours) 	& \textbf{68.4}&	\textbf{71.9}&	\textbf{80.0}&	77.2&	\textbf{59.5}&	73.8&	75.6&	\textbf{71.7}&	\textbf{70.7}&	\textbf{72.2}&	77.7&	67.1&	\textbf{72.7}&	\textbf{70.8}&	\textbf{69.3} \\
			\bottomrule
	\end{tabular}}
\end{sidewaystable}

\begin{figure}[!t]
	\centering
	\includegraphics[width=0.65\textwidth]{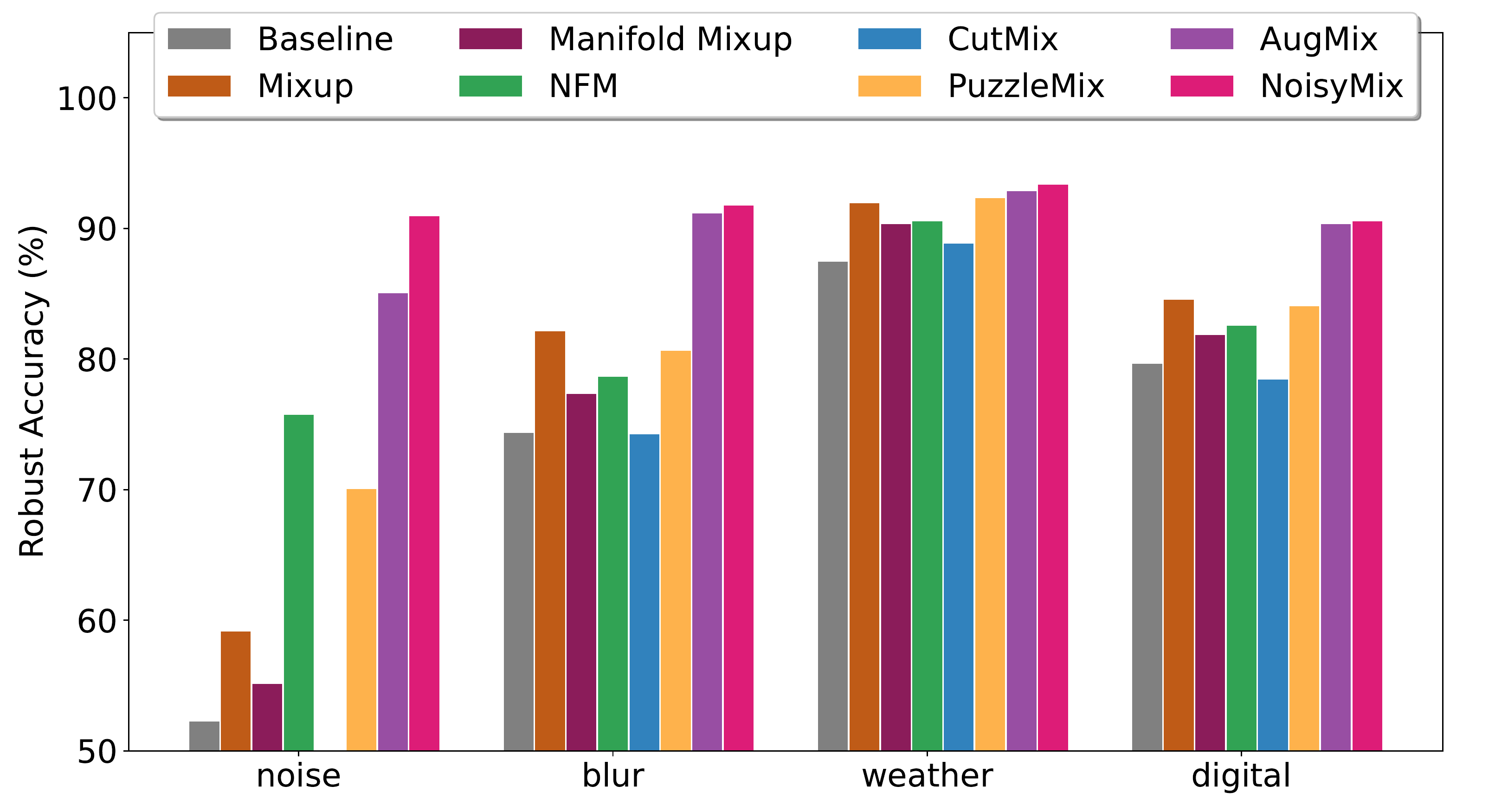}
	\vspace{-0.4cm}
	\caption{Average robust accuracy of Wide-ResNet-28x2 models, trained on CIFAR-10, for the 4 categories noise, blur, weather, and digital distortions. \textit{NoisyMix} dominates in 1 out of the 4 categories, while the robustness to blur, weather, and digital corruptions is nearly on par with AugMix.} 
	\label{fig:cifar10_wrs_type}
\end{figure}

\begin{figure}[!t]
	\centering
	\includegraphics[width=0.65\textwidth]{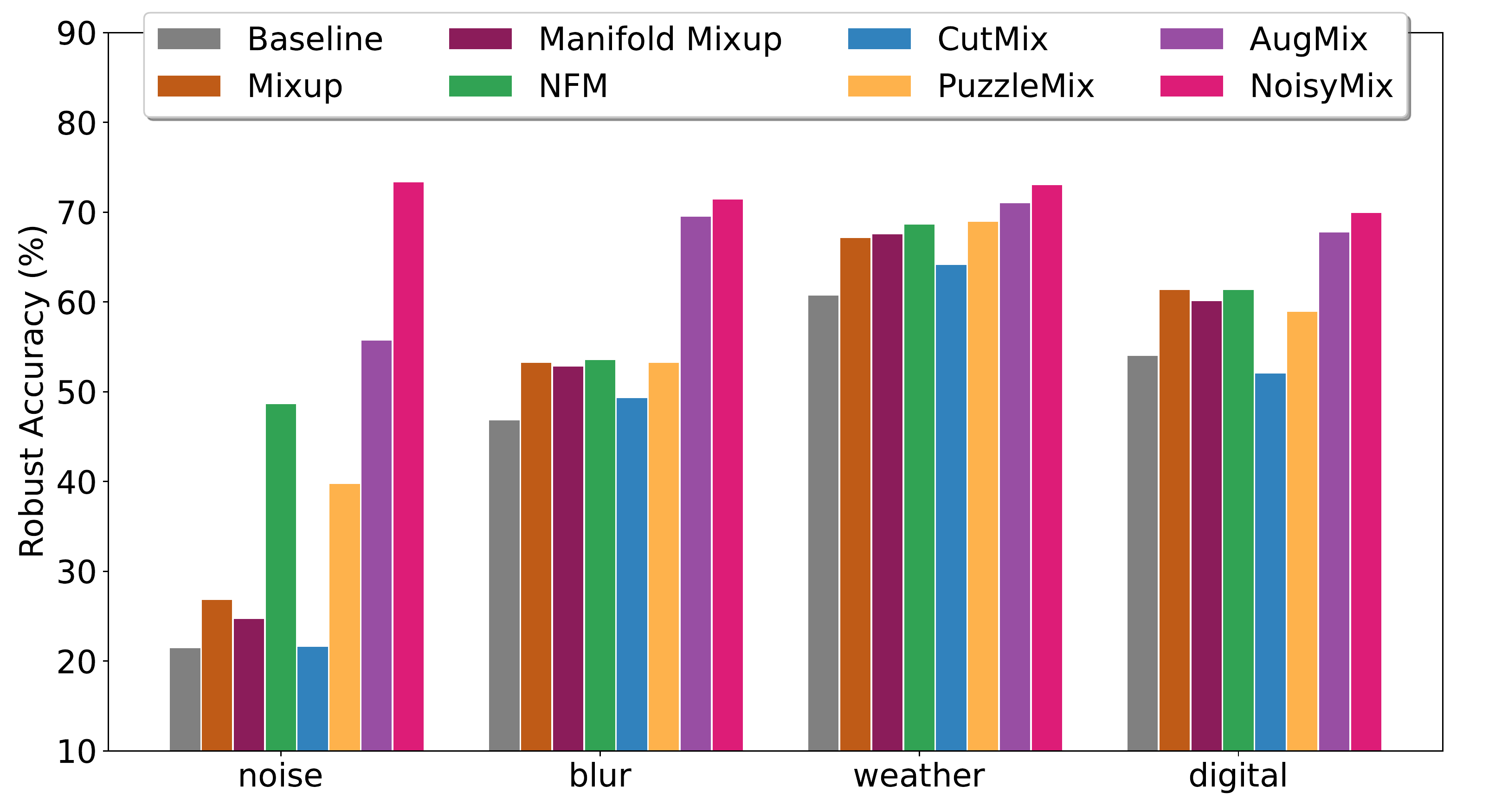}
	\vspace{-0.4cm}
	\caption{Average robust accuracy of Wide-ResNet-28x2 models, trained on CIFAR-100, for the 4 categories noise, blur, weather, and digital distortions. \textit{NoisyMix} dominates in 4 out of the 4 categories.} 
	\label{fig:cifar100_wrs_type}
\end{figure}

Table~\ref{tab:results_wrs_cifar10_c} shows the detailed robustness properties of Wide-ResNet-28x2 models evaluated on CIFAR-10-C, and Table~\ref{tab:results_wrs_cifar100_c} results for CIFAR-100-C, respectively. \textit{NoisyMix} benefits from the wide architecture and can improve the robustness as compared to the ResNet-18 models. Again, it can be seen that \textit{NoisyMix} significantly improves the robustness to noise perturbations as well as to weather corruptions.
Figure~\ref{fig:cifar10_wrs_type}, and Figure~\ref{fig:cifar100_wrs_type} show the average robustness accuracy for the 4 meta categories (i.e., noise, weather, blur, and digital corruptions) for CIFAR-10-C and CIFAR-100-C, respectively.

\subsection{Calibration Results for ImageNet-R} 

State-of-the-art neural networks tend to be biased towards high accuracy prediction, while being poorly calibrated. In turn, these models tend to be less reliable and often also show to lack fairness.
This can lead to poor decision making in mission critical applications, e.g., medical imaging. In such applications it might be not sufficient to just compute a score $s(x)$ that can then be used to rank the input from the most probable member to the least probable member of a class $c$. Rather it is desirable to obtain an reliable estimate for the class membership probability that can be assigned an example-dependent misclassifications cost~\cite{zadrozny2001obtaining}.
Typically, these probability estimates are obtained by squashing the activations of the output layer through a normalized exponential function that is known as the softmax function.

One hypothesis that explains why DNNs (such as ResNets) are poorly calibrated is their high capacity, i.e., these models have the ability to learn highly-specific features that overfit to the source dataset. This phenomenom is shown by \cite{guo2017calibration}, where calibration error increases as the number of filters per layer increases.
We show that data augmentation can mitigate this effect and produce better calibrated models. 
Table~\ref{tab:results_imagenet_r_calibration} shows the RMS calibration error and the area under the response rate accuracy curve (AURRA), evaluated on IamgeNet-R, as metrics to measure how well a model is calibrated.
\textit{NoisyMix} considerably improves the calibration of the ImageNet model, as compared to standard training or training with other data augmentation methods.
We believe that this is explained by the fact that \textit{NoisyMix} ``wash-out'' features that are highly task specific.

\begin{table*}[!b]
	\caption{Calibration errors  measured by the RMS calibration error (lower values are better), and the area under the response rate accuracy curve (higher values are better).}
	\label{tab:results_imagenet_r_calibration}
	\centering
	\scalebox{0.85}{
		\begin{tabular}{lccccccccc}
			\toprule
			& RMS Calibration Error ($\downarrow$\%) & AURRA ($\uparrow$\%)\\			
			\midrule 
			Baseline ResNet-50								& 19.7 & 64.6\\
			Adversarial Trained~\cite{NEURIPS2020_24357dd0} & 8.9 & 68.8 \\
			Stylized ImageNet~\cite{geirhos2018imagenet} 	& 16.2 & 69.7\\
			AutoAugment~\cite{cubuk2019autoaugment}			& 19.9 & 67.6\\
			
			Mixup~\cite{zhang2018mixup} 					& 17.5 & 68.6 \\
			Manifold Mixup~\cite{verma2019manifold} 		& 5.0 & 68.7\\
			CutMix~\cite{yun2019cutmix}						& 13.6 & 63.2 \\
			Puzzle Mix~\cite{kim2020puzzle}					& 15.8 & 68.8\\
			AugMix~\cite{hendrycks2020augmix} 				& 14.5 & 70.1\\
			
			\textit{NoisyMix}  (ours) 						& \textbf{3.9} & \textbf{74.8}\\
			
			\bottomrule
	\end{tabular}}
\end{table*}

\subsection{Additional ImageNet-P and CIFAR-10-P Results}

\begin{table*}[!b]
	\caption{Detailed results for ResNet-50s trained with various data augmentation schemes and evaluated on ImageNet-P. The values indicate the flipping rates as a function of the 10 perturbation types (lower values are better). The mean flip rate (mFR) summarizes the results and shows that \textit{NoisyMix}  improves perturbation stability by approximately 5\%.}
	\label{tab:results_imagenet_p}
	\centering
	\scalebox{0.8}{
		\begin{tabular}{l|cc|cc|cc|cccc|cccccc}
			\toprule
			& \multicolumn{2}{|c|}{Noise} & \multicolumn{2}{|c|}{Blur} & \multicolumn{2}{|c|}{Weather} & \multicolumn{4}{|c|}{Digital} &  \\	
			& White & Shot & Motion & Zoom & Snow & Bright & Translate & Rotate & Tilt & Scale & mFR\\
			\hline
			Baseline 										& 59.4 &	57.8 & 64.5 & 72.1 & 63.2 & 61.9 & 44.2 & 51.9 & 56.9 & 48.1 & 58.0\\
			Adversarial Trained & 25.6&	40.1&	\textbf{26.9}&	\textbf{33.1}&	\textbf{18.3}&	60.6&	25.7&	31.3&	30.5&	40.4 & 33.3\\
			Stylized ImageNet & 31.1&	28.8&	41.8&	58.1&	43.1&	50.3&	36.7&	41.2&	42.9&	42.1& 54.4\\
			Fast AutoAugment & 53.7&	49.9&	63.8&	79.3&	66.5&	58.2&	43.2&	49.0&	57.7&	43.6& 56.6 \\
			Mixup & 47.6&	47.2&	60.8&	73.9&	58.9&	62.1&	51.2&	55.0&	61.1&	46.3& 56.4 \\
			Manifold Mixup 	& 49.1&	46.2&	66.1&	72.4&	62.6&	59.9&	46.5&	53.7&	58.8&	45.2& 56.0\\
			CutMix & 54.7&	52.7&	69.4&	77.4&	73.1&	59.2&	43.2&	51.6&	59.1&	46.1& 58.6 \\
			Puzzle Mix & 50.1&	48.3&	64.6&	72.5&	63.3&	56.7&	44.3&	50.4&	57.9&	46.6 & 55.5\\
			AugMix 			& 40.9&	37.8&	30.0&	54.8&	37.3&	46.1&	26.9&	32.3&	36.4&	33.8 & 37.6\\
			\textit{NoisyMix}  (ours) 	& \textbf{25.1}&	\textbf{22.8}&	30.2&	41.5&	32.4&	\textbf{34.1}&	\textbf{20.5}&	\textbf{24.5}&	\textbf{27.7}&	\textbf{26.3} & \textbf{28.5}\\
			
			\bottomrule
	\end{tabular}}
\end{table*}

\begin{table*}[!t]
	\caption{Detailed results for ResNet-18s trained with various data augmentation schemes and evaluated on CIFAR-10-P. The values indicate the flipping rates as a function of the 10 perturbation types (lower values are better). The mean flip rate (mFR) summarizes the results and shows that \textit{NoisyMix}  improves perturbation stability by approximately 1\%.}
	\label{tab:results_cifar_p}
	\centering
	\scalebox{0.8}{
		\begin{tabular}{l|cc|cc|cc|cccc|cccccc}
			\toprule
			& \multicolumn{2}{|c|}{Noise} & \multicolumn{2}{|c|}{Blur} & \multicolumn{2}{|c|}{Weather} & \multicolumn{4}{|c|}{Digital} &  \\	
			& White & Shot & Motion & Zoom & Snow & Bright & Translate & Rotate & Tilt & Scale & mFR\\
			\hline
			Baseline 					& 44.6&	27.6&	9.2&	0.4&	2.5&	0.5&	2.1&	2.8&	0.8&	3.1&	9.4\\
			Mixup 						& 26.3&	17.2&	6.1&	0.4&	1.8&	0.6&	2.2&	2.4&	0.8&	2.5&	6.0 \\
			Manifold Mixup 				& 33.4&	21.6&	7.8&	0.4&	2.1&	0.6&	2.2&	2.6&	0.8&	2.9&	7.4\\
			NFM 						& 11.5&	7.5&	6.8&	0.3&	1.5&	\textbf{0.4}&	1.8&	2.0&	0.6&	2.2&	3.5\\			
			CutMix 						& 72.3&	50.8&	9.1&	0.5&	1.9&	0.6&	1.9&	2.3&	0.8&	3.1&	14.3 \\
			Puzzle Mix 					& 16.1&	11.2&	6.3&	0.3&	1.4&	\textbf{0.4}&	2.3&	1.8&	0.6&	2.5&	4.3\\
			AugMix 						& 11.4&	7.1&	2.4&	0.2&	1.0&	\textbf{0.4}&	1.9&	1.2&	0.5&	1.8&	2.8\\
			\textit{NoisyMix}  (ours) 	& \textbf{6.6}&	\textbf{4.5}&	\textbf{2.3}&	\textbf{0.1}&	\textbf{0.9}&	\textbf{0.4}&	\textbf{1.6}&	\textbf{1.0}&	\textbf{0.4}&	\textbf{1.4}&	\textbf{1.9}\\
			
			\bottomrule
	\end{tabular}}
\end{table*} 

\begin{table*}[!t]
	\caption{Detailed results for Wide-ResNet-28x2 trained with various data augmentation schemes and evaluated on CIFAR-10-P. The values indicate the flipping rates as a function of the 10 perturbation types (lower values are better). The mean flip rate (mFR) summarizes the results and shows that \textit{NoisyMix}  improves perturbation stability by approximately 5\%.}
	\label{tab:results_cifar_p_wrs}
	\centering
	\scalebox{0.8}{
		\begin{tabular}{l|cc|cc|cc|cccc|cccccc}
			\toprule
			& \multicolumn{2}{|c|}{Noise} & \multicolumn{2}{|c|}{Blur} & \multicolumn{2}{|c|}{Weather} & \multicolumn{4}{|c|}{Digital} &  \\	
			& White & Shot & Motion & Zoom & Snow & Bright & Translate & Rotate & Tilt & Scale & mFR\\
			\hline
			Baseline 					& 42.7&	27.6&	8.2&	0.3&	2.1&	0.4&	1.4&	2.4&	0.6&	2.6&	8.8\\
			Mixup 						& 30.9&	18.5&	5.7&	0.3&	1.5&	0.5&	1.6&	1.8&	0.7&	2.1&	6.4\\
			Manifold Mixup 				& 43.8&	26.6&	7.0&	0.3&	1.6&	0.4&	1.3&	1.9&	0.6&	2.2&	8.6\\
			NFM 						& 13.1&	7.8&	8.0&	0.3&	1.4&	0.4&	1.4&	2.0&	0.6&	2.1&	3.7\\			
			CutMix 						& 78.0&	60.1&	9.6&	0.5&	1.9&	0.6&	1.5&	2.4&	0.8&	3.2&	15.9 \\
			Puzzle Mix 					& 14.6&	9.5&	7.1&	0.4&	1.3&	0.5&	1.8&	1.9&	0.7&	2.5&	4.0\\
			AugMix 						& 11.2&	6.6&	2.1&	0.2&	0.9&	\textbf{0.3}&	1.3&	1.0&	0.4&	1.3&	2.5\\
			\textit{NoisyMix}  (ours) 	& \textbf{6.3}&	\textbf{4.2}&	\textbf{2.0}&	\textbf{0.1}&	\textbf{0.7}&	\textbf{0.3}&	\textbf{1.0}&	\textbf{0.8}&	\textbf{0.3}&	\textbf{1.1}&	\textbf{1.7}\\
			
			\bottomrule
	\end{tabular}}
\end{table*}

\paragraph{ImageNet-C Results.}
Table~\ref{tab:results_imagenet_p} shows the flip rates with respect to each perturbation type for the different ResNet-50 models trained on ImageNet-1k.
\textit{NoisyMix} dominates on 6 out of the 10 corruption types, resulting in the lowest mean flip rate overall. While \textit{NoisyMix} shows a clear advantage compared to the model trained with AugMix, the adversarial trained model is performing nearly as good as \textit{NoisyMix} on this task. This is surprising, since the adversarial trained model shows a poor performance on the ImageNet-C task. However, this shows the many different facets of robustness and supports the claim that there is no single metric for measuring robustness. 

Table~\ref{tab:results_cifar_p} and~\ref{tab:results_cifar_p_wrs} show results for ResNet-18 and Wide-ResNet-28x2 models trained on CIFAR-10 and evaluated on CIFAR-10-P. Here, the models trained with \textit{NoisyMix} clearly dominate. Note, that we do not normalize the flip rates here.

\end{document}